\documentclass[11pt]{article}

\usepackage{setspace}
\usepackage{cprotect}
\usepackage{fullpage}
\usepackage{color}
\usepackage{cite}
\usepackage{fixltx2e}
\usepackage[cmex10]{amsmath}
\usepackage{amssymb}
\usepackage{array}
\usepackage{wasysym}
\usepackage{dsfont}
\usepackage[english]{babel}                         
\usepackage[T1]{fontenc}
\usepackage{mathtools}
\usepackage{amssymb} 
\usepackage{enumerate}
\usepackage{bbm}
\usepackage{epsfig,syntonly}
\usepackage{epstopdf}
\usepackage{graphicx}
\usepackage{latexsym,fancyhdr,bm}
\usepackage{subfig}
\usepackage{caption}
\usepackage{url}
\usepackage{amsthm}
\usepackage{xfrac}
\newcommand{\norm}[1]{\left\lVert#1\right\rVert}

\theoremstyle{plain}

\newtheorem{theorem}{Theorem}
\newtheorem{lemma}{Lemma}

\newtheorem{definition}{Definition}

\newtheorem{corollary}{Corollary}
\newtheorem{conjecture}{Conjecture}

\newtheorem*{theorem*}{Theorem}
\newtheorem*{lemma*}{Lemma}
\newtheorem*{proposition*}{Proposition}
\newtheorem*{definition*}{Definition}
\newtheorem*{example*}{Example}

\newtheorem*{corollary*}{Corollary}

\newtheoremstyle{custom}
{} % Space above
{} % Space below
{\rm } % Body font
{} % Indent amount
{\bfseries} % Theorem head font
{:} % Punctuation after theorem head
{.25em} % Space after theorem head
{} % Theorem head spec (can be left empty, meaning `normal')
\theoremstyle{custom}

\newtheorem{remark}{Remark}
\newtheorem*{remark*}{Remark}

\def\sT{{\mathsf T}}

\def\cS{{\mathcal S}}
\def\cA{{\mathcal A}}
\def\cS{{\mathcal S}}
\def\cP{{\mathcal P}}
\def\cN{{\mathcal N}}
\def\cW{{\mathcal W}}
\def\cD{{\mathcal D}}
\def\Var{{\rm Var}}

\def\<{\langle}
\def\>{\rangle}
\def\eps{\epsilon}

\def\reals{{\mathbb R}}

\def\<{\langle}
\def\>{\rangle}

\def\reals{{\mathbb R}}

\def\bv{{\boldsymbol v}}
\def\bx{{\boldsymbol x}}
\def\bw{{\boldsymbol w}}
\def\hbw{\hat{\boldsymbol w}}
\def\bT{{\boldsymbol T}}
\def\by{{\boldsymbol y}}

\def\hy{\hat{y}}

\def\bA{{\boldsymbol A}}
\def\bB{{\boldsymbol B}}

\def\bW{{\boldsymbol W}}

\def\bI{{\boldsymbol I}}

\def\b0{{\boldsymbol 0}}

\def\bu{{\boldsymbol u}}
\def\bg{{\boldsymbol g}}
\def\bv{{\boldsymbol v}}

\def\normal{{\sf N}}

\def\E{{\mathbb E}}

\makeatletter
\let\l@ENGLISH\l@english
\makeatother

\title{On the Connection Between Learning Two-Layer Neural Networks and Tensor Decomposition}

\author{Marco~Mondelli\thanks{Department of Electrical Engineering,
    Stanford University}
 \;\;\; and\;\;\; Andrea Montanari\thanks{Department of Electrical Engineering and Department of Statistics, Stanford University}}

\begin{document}

\maketitle
%------------------------------------------------------------------------------------------------------------------
\begin{abstract}
We establish connections between the problem of learning a two-layer neural network and tensor decomposition. We consider a model with feature vectors $\bx \in \mathbb R^d$,
$r$ hidden units with weights $\{\bw_i\}_{1\le i \le r}$ and output $y\in \mathbb R$, i.e., $y=\sum_{i=1}^r \sigma( \bw_i^{\sT}\bx)$,
with activation functions given by low-degree polynomials.
In particular, if $\sigma(x) = a_0+a_1x+a_3x^3$, we prove that no polynomial-time learning algorithm can outperform the trivial predictor that assigns to 
each example the response variable $\E(y)$, when $d^{3/2}\ll r\ll d^2$. 
Our conclusion holds for a `natural data distribution', namely standard Gaussian feature vectors $\bx$, and output distributed according 
to a two-layer neural network with random isotropic weights, and under a certain complexity-theoretic assumption on tensor decomposition. 
Roughly speaking, we assume that no polynomial-time algorithm can substantially outperform current methods for tensor decomposition 
based on the sum-of-squares hierarchy.

We also prove generalizations of this statement for higher degree polynomial activations, and non-random weight vectors.
Remarkably, several existing algorithms for learning two-layer networks with rigorous guarantees are
based on tensor decomposition. Our results support the idea that this is indeed the core computational difficulty in learning such networks,
under the stated generative model for the data. As a side result, we show that under this model learning the network requires accurate learning of its weights, a 
property that does not hold in a more general setting.
\end{abstract}

%\begin{IEEEkeywords}
%
%\end{IEEEkeywords}

%----------------------------------------------------
\section{Introduction and Main Results} \label{sec:intro}
%----------------------------------------------------

Let $\{(\bx_i,y_i)\}_{1\le i\le n}$ be $n$ data points where, for each $i$, $\bx_i\in\reals^d$ is a feature vector and $y_i\in\reals$ is a response
variable or label. The simplest neural network attempts to fit these data using the model
\begin{equation}\label{eq:defhaty}
\hat{y}(\bx;\hbw) = \sum_{i=1}^r \sigma(\left \langle \bx, \hat{\bw}_i\right \rangle)\, .
\end{equation}
Here $\sigma:\reals\to \reals$ is a non-linear activation function, and $\hbw = (\hbw_i)_{i\le r}$, where $\hbw_1,\dots,\hbw_r\in\reals^d$ are model parameters
(weight vectors). In the following, we will often omit the argument $\hbw$ from $\hy$.
Let us emphasize that this is a deliberately oversimplified neural network model: $(i)$~It only includes one hidden layer of $r$ units (neurons);
$(ii)$~The output unit is linear (it takes a linear combination of the hidden units); 
$(iii)$~The hidden units have no offset or output weight. Since our main results are negative (computational hardness), we are not too concerned with  such 
simplifications. For instance, it is unlikely that adding a non-linear output unit can reduce the problem hardness.

Throughout this paper, we will assume the data to be i.i.d. with common distribution $\cD$, namely
$(\bx_i,y_i)\sim\cD$. A rapidly growing literature develops algorithms and rigorous guarantees to learn such a model, see e.g. 
 \cite{janzamin2015beating, soudry2016no, soltanolkotabi2018theoretical, safran2016quality, freeman2016topology, ge2017learning, zhong2017recovery}
and the brief overview in Section \ref{subsec:related}. These papers analyze the landscape of empirical risk minimization for 
the model (\ref{eq:defhaty}), or its variants. Under suitable assumptions on the data distribution $\cD$ (as well as the parameters $d,r,n$)
they develop algorithms that are guaranteed to recover the weights $\hbw_1,\dots,\hbw_r$ with small training error.

In this paper we consider the complementary question, and use a reduction from tensor decomposition 
to provide evidence  that --in certain regimes, and for certain data distributions $\cD$-- the model (\ref{eq:defhaty}) cannot be learnt in polynomial time. 
Let us emphasize two important aspects of our results:
\begin{itemize}
\item Our impossibility results are entirely computational, and do not depend on the data distribution $\cD$. Indeed, they hold
even if we have access to an infinite sample. (More accurately, they hold under a stronger model that allows us to compute expectations with respect to $\cD$).
\item Earlier work has proven computational hardness for simpler problems than the neural network (\ref{eq:defhaty}). For instance, \cite{daniely2016complexity}
proves hardness for learning a single linear classifier. However these proofs are based on the construction of special distributions  $\cD$ that
are are unknown to the learner. Here instead we consider a `natural' class of distributions $\cD$ that is in fact normally assumed in works
estabilishing positive guarantees. This point of view  is similar to the one recently developed in \cite{shamir2016distribution} although our methods and results are quite different.
\end{itemize}

As mentioned above, our results are conditional on a complexity-theoretic assumption for tensor decomposition, i.e. the problem of recovering the weights $\{\bw_i\}_{1\le i\le r}$ 
given access to the $k$-th order tensor $\bT^{(k)} = \sum_{i=1}^r \bw_i^{\otimes k}$. We state this assumption explicitly below, for the case of tensors of order $k=3$.
\begin{conjecture}[$\epsilon$-Hardness of 3-Tensor Decomposition]\label{conj:tensor}
The following holds for some $\eps_0>0$, and all $\delta>0$.
Define a distribution $\cW_{d,r}$ over the weights  $\bw = (\bw_i)_{1\le i \le r}\in (\reals^{d})^r$, by letting
\begin{equation}
\bw_i =\frac{\bg_i -\frac{1}{r}\sum_{j=1}^r \bg_j}{\norm{\bg_i -\frac{1}{r}\sum_{j=1}^r \bg_j}}, \qquad \forall\,i\in [r],\label{eq:WeightDistr}
\end{equation}
where $\{\bg_i\}_{1\le i \le r}\sim_{\rm i.i.d.}\normal (\b0_d, \bI_d/d)$.  Set $\bT(\bw_1, \ldots, \bw_r)= \sum_{i=1}^r \bw_i^{\otimes 3}$. 
Assume $r = r(d)\ge d^{(3/2)+\delta}$ and $\eps<\eps_0$.

Then there is no algorithm  $\cA$ that, given as input $\bT(\bw)$, with $\bw= (\bw_i)_{1\le i \le r}\sim\cW_{d,r}$ fulfills the following two properties:
\begin{itemize}
\item[\sf (P1)] $\cA$ outputs $\{\hat{\bw}_i\}_{1\le i\le r}$ of unit norm such that, with probability at least $1/2$,
 for some $i, j\in [r]$, $|\left \langle \bw_i, \hat{\bw}_j\right \rangle| \ge \epsilon$.
\item[\sf (P2)] $\cA$ has complexity bounded by a polynomial in $d$.
\end{itemize}
\end{conjecture}

Tensor decomposition has been studied by a number of authors, and the best known algorithms are based on (or match the guarantees of) the sum-of-squares (SoS) hierarchy
 \cite{hopkins2016fast, ma2016polynomial, schramm2015low}. The above assumption amounts to conjecturing that no algorithm can beat SoS for this 
problem\footnote{An important technical remark is that we assume it is impossible to estimate even a single component of $\bT$. This is motivated by the remark that
 in all existing algorithmic approaches for tensor decomposition, the problems of learning a single component and of learning all components are either 
both solvable or both unsolvable, e.g., see \cite{ma2016polynomial, schramm2017fast}. It is also easy to see that they are equivalent if we demand \emph{exact}
reconstruction of the weights.}.
We limit ourselves to noticing that SoS appears to capture computational boundaries in a number of similar statistical problems 
\cite{barak2014sum,hopkins2015tensor,barak2015dictionary,barak2016nearly,hopkins2017power}.
\begin{theorem}\label{th:hardgen_Gnew}
Let  $\sigma(x) = a_0+a_1 x+a_3x^3$ for some $a_0, a_1,a_3\in \mathbb R$ and 
denote by $\cN(d,r)$ the set of functions $\hy(\,\cdot\, ;\hbw) :\reals^d\to\reals$ of the form (\ref{eq:defhaty}) where $\|\hbw_1\|_2=\dots=\|\hbw_r\|_2=1$.
Assume $r=r(d)$ to be such that $d^{(3/2)+\delta}\le  r\le d^{2-\delta}$ for some $\delta>0$. Then, under Conjecture \ref{conj:tensor}, there
 exists $\eta(r,d)\to 0$ as $d\to\infty$ such that the following holds.

Let $\bw = (\bw_j)_{j\le r}\sim\cW_{d,r}$ be random weights, see Eq.~(\ref{eq:WeightDistr}).
Consider data $\{(\bx_i,y_i)\}_{i\le n}$ with common distribution $\cD$ defined by $\bx_i\sim \normal(\b0_d,\bI_d)$
and $y_i =y(\bx_i)= \hy(\bx_i;\bw)$, with $\hy(\bx;\bw)$ given by \eqref{eq:defhaty}. In particular,
\begin{equation}
\min_{\hy(\,\cdot\,) \in\cN(d,r)}\E_{\cD}\left\{|y(\bx)- \hy(\bx)|^2\right\}  =0\, .
\end{equation}

However, for any polynomial-time algorithm $\cP$ that takes as input $\{(\bx_i,y_i)\}_{i\le n}$ and returns a function $\hy_{\cP}\in \cN(d,r)$, we have that
\begin{equation}\label{eq:lbgen_cor_G}
\E_{\cD}\big\{|y(\bx)-\hy_{\cP}(\bx)|^2\big\}\ge\Var\left\{y(\bx)\right\}\, (1-\eta(r, d))\, ,
\end{equation}
with high probability with respect to $\bw \sim\cW_{d,r}$. 
\end{theorem}

 A few remarks are in order.
\begin{remark}
The right-hand side of Eq.~(\ref{eq:lbgen_cor_G}) is the risk of a trivial model that always predicts $y$ with its expectation. 
Hence, Theorem \ref{th:hardgen_Gnew} implies that, under the data distribution $\cD$, no polynomial algorithm can 
predict the response better than a trivial predictor that assigns to each example the same response $\E(y)$.
Notice that this lower bound is independent of $n$, and in fact we prove it under a more powerful model, whereby the algorithm $\cP$
is given access to an oracle that computes expectations with respect to $\cD$.  

On the other hand, under unbounded computation,  it is possible to find a neural network of the form \eqref{eq:defhaty},
with zero test error.
\end{remark}

\begin{remark}
A large part of the theoretical literature adopts the same model of the above theorem, namely
random Gaussian features $\bx\sim \normal(\b0_d,\bI_d)$, and data generated according to a two-layer network with 
random weights, see e.g. \cite{janzamin2015beating,ge2017learning,zhong2017recovery,soltanolkotabi2018theoretical}.
Our theorem implies that within the assumptions of these papers, $r\ll d^{3/2}$ is a computational hardness barrier
(under the stated conjecture on tensor decomposition).

Note that several of these papers use tensor decomposition procedures as a key subroutine (typically to initialize the
weights before a gradient descent phase). Theorem \ref{th:hardgen_Gnew}
implies that the appearance of tensor decomposition in these algorithms is a consequence of a fundamental connection between the two problems.
\end{remark}

In the rest of this introduction we provide a brief overview of related work.
We then present our technical contributions. In Section \ref{sec:gen}, we show that, if we cannot estimate the weights $\{\bw_i\}_{1\le i \le r}$ accurately, then 
the error $\E\{|y(\bx)-\hy(\bx)|^2\}$, typically called generalization error\footnote{The term `generalization error' is often used interchangeably with `risk' and it refers to the expected loss of a prediction rule also in the realizable case, see \cite{bousquet2002stability, zhong2017recovery} and \cite[pp. 34-35]{shalev2014understanding}.}, of the predictor $\hat{y}(\bx)$ is close to that of a trivial predictor. We prove this result in two separate settings:
for deterministic and for random weights $(\bw_i)_{i\le r}$.
In Section \ref{sec:learn}, we present reductions from the problem of tensor decomposition to the problem of estimating the weights $\{\bw_i\}_{1\le i \le r}$
in the two-layer neural network model.
By combining these two results,  in Section \ref{sec:disc} we present  reductions from the problem of tensor decomposition to the problem of learning a two-layer neural network
with small error $\E\{|y(\bx)-\hy(\bx)|^2\}$. These results generalize Theorem \ref{th:hardgen_Gnew} in two directions: we consider non-random weights, and 
 a broader set of polynomial activation functions $\sigma(\,\cdot\,)$. Finally, in Section \ref{sec:simu}, we present numerical experiments supporting our theoretical findings. 

In summary, we consider a popular model for theoretical research (random two-layer neural network with Gaussian feature vectors)
and show that: $(i)$ learning in this model requires accurate weight estimation;
 and $(ii)$ the latter requires solving a tensor decomposition problem, which is computationally expensive.
A promising direction of research would be to understand whether these conclusions can be avoided by considering different generative models.

%----------------------------------------------------
\subsection{Related Work} \label{subsec:related}
%----------------------------------------------------

Several recent papers provide recovery guarantees for neural network models, and what follows is a necessarily incomplete overview. 
In \cite{arora2014provable}, the weights are assumed to be sparse and random, and the proposed algorithm learns almost all the models in this class 
with polynomial sample complexity and computational complexity. In \cite{brutzkus2017globally}, the authors consider a two-layer neural network with convolutional 
structure, no overlap\footnote{The filter of the convolutional neural network is applied to non-overlapping parts of the input vector.}, and ReLU activation function. It is shown that learning is NP-complete in the worst case, but gradient descent converges to the global optimum in polynomial time when the input distribution is Gaussian. 
A similar positive result, i.e., convergence to the global optimum of gradient descent  with polynomial complexity and Gaussian input, 
is proved in \cite{tian2017symmetry}. In this work, the author considers a two-layer neural network model of the form \eqref{eq:defhaty}, where $\sigma$ is a 
ReLU activation function and the weights $\{\bw_i\}_{1\le i \le r}$ are orthogonal (which implies that $r\le d$). However, \cite{tian2017symmetry} requires a good initialization 
and does not discuss initialization methods. In \cite{panigrahy2018convergence}, the authors design an activation function that guarantees provable learning, 
but the proposed algorithm runs in $d^{O(d)}$. In \cite{sedghi2015provable}, the subspace spanned by the weight matrix is provably recovered with a tensor 
decomposition  algorithm, and the weights can also be recovered under an additional sparsity assumption. The works 
\cite{brutzkus2017globally, tian2017symmetry, sedghi2015provable} consider only the population risk and do not give bounds on the sample complexity. 
The paper \cite{janzamin2015beating} presents a tensor based algorithm that learns a two-layer neural network with  sample complexity of order $d^3\cdot {\rm poly}(r)/\varepsilon^2$, where $\varepsilon$ is the precision. In \cite{zhong2017recovery}, a tensor initialization algorithm is combined with gradient descent to obtain a procedure with sample complexity of order $d\cdot {\rm poly}(r)\cdot \log(1/\varepsilon)$ and computational complexity $n\cdot d\cdot {\rm poly}(r)\cdot \log(1/\varepsilon)$, where $n$ is the number of samples and it is assumed that $r\le d$. The connection between tensors and neural networks is also studied in \cite{ge2017learning}.

As mentioned above, several hardness results are available for training neural networks or even simple
linear classifiers \cite{blum1989training, bartlett1999hardness, kuhlmann2000hardness, vsima2002training, daniely2016complexity}. However, these
results rely on special constructions of the distribution $\cD$. In contrast here, we consider a specific class of distributions that has been 
frequently studied in the algorithms literature, in order to estabilish rigorous guarantees.
Similar in spirit to our results is the recent work of Ohad Shamir \cite{shamir2016distribution}
which considers data generated according to the model (\ref{eq:defhaty}) with smooth distributions of the feature vectors $\bx$,
and periodic activation functions (while we consider low-degree polynomials). Apart from technical differences in the model definition,
our results are different and complementary to the ones of  \cite{shamir2016distribution}. While \cite{shamir2016distribution} analyzes
a specific class of `approximate gradient' algoritms, we prove a general hardness result, conditional on a complexity-theoretic assumption.

%----------------------------------------------------
\section{Preliminaries} \label{sec:prel}
%----------------------------------------------------

%----------------------------------------------------
\subsection{Notation and System Model} \label{subsec:sys}
%----------------------------------------------------

Let $[n]$ be a shorthand for $\{1, \ldots, n\}$. Let $\b0_n$ and $\boldsymbol 1_n$ denote the vector consisting of $n$ 0s and $n$ 1s, respectively, and let $\bI_n$ denote the $n\times n$ identity matrix. Given a vector $\bx\in \mathbb R^n$, we let $x(i)$ be its $i$-th element, where $i\in [n]$, and $\norm{\bx}$ be its $\ell_2$ norm. Given a matrix $\bA$, we let $\bA^{\sT}$ be its transpose, ${\rm Tr}(\bA)$ be its trace, $\norm{\bA}_F$ be its Frobenius norm, and $\norm{\bA}_{\rm op}$ be its operator norm. We use $\bA\otimes \bB$ to denote the Kronecker product of $\bA$ and $\bB$, and $\bA^{\otimes k}$ as a shorthand for $\bA\otimes \cdots \otimes \bA$, where $\bA$ appears $k$ times. We also set $\bA^{\otimes 0}=1$. Given two $k$-th order tensors $\bx, \by \in ({\mathbb R}^d)^{\otimes k}$, we let $\left \langle \bx, \by \right \rangle = \sum_{i_1,  \ldots, i_k=1}^d x(i_1, \ldots, i_k)  \cdot y(i_1, \ldots, i_k)$ be their scalar product. Given a $k$-th order tensor $\bx \in ({\mathbb R}^d)^{\otimes k}$, we let $\norm{\bx}_F=\sqrt{\left \langle \bx, \bx \right \rangle}$ be its Frobenius norm. Given an integer $k$, we denote by ${\rm par}(k)$ its parity, i.e., we set ${\rm par}(k)$ to $0$ if $k$ is even and to $1$ if $k$ is odd. Given a polynomial $f$, we denote by ${\rm deg}(f)$ its degree. If $f$ is either even or odd, we denote by ${\rm par}(f)$ its parity, i.e., we set ${\rm par}(f)$ to $0$ if $f$ is even and to $1$ if $f$ is odd. Given a function $\sigma$ in the weighted $L^2$ space\footnote{$L^2(\mathbb R, e^{-x^2/2})= \left\{\sigma : \int_{\mathbb R} |\sigma(x)|^2 e^{-x^2/2}\,{\rm d}x<\infty\right\}$.} $L^2(\mathbb R, e^{-x^2/2})$, we denote by $\hat{\sigma}_k$ its $k$-th Hermite coefficient. It is helpful to write explicitly the formulas to compute $\hat{\sigma}_1$ and $\hat{\sigma}_2$:
\begin{equation}\label{eq:her12}
\hat{\sigma}_1 = {\mathbb E}_{G\sim \normal(0, 1)}\left\{G\cdot \sigma(G)\right\},\qquad
\hat{\sigma}_2 = \frac{1}{\sqrt{2}}{\mathbb E}_{G\sim \normal(0, 1)}\left\{(G^2-1)\sigma(G)\right\}.
\end{equation}

Throughout the paper, we consider a two-layer neural network with input dimension $d$ and $r$ hidden nodes with weights $\bw = (\bw_i)_{1\le i \le r}\in (\reals^{d})^r$. We denote the input by $\bx\in \mathbb R^d$ and the output by $y(\bx; \bw)\in \mathbb R$, which is defined by
\begin{equation}\label{eq:defy}
y(\bx;\bw) = \sum_{i=1}^r \sigma(\left \langle \bx, \bw_i\right \rangle)\, .
\end{equation}
We will often omit the argument $\bw$ from $y$. Given $n$ samples from the neural network, we obtain the estimates $\{\hat{\bw}_i\}_{1\le i \le r}$ on the weights $\{\bw_i\}_{1\le i \le r}$, which allows us to construct $\hat{y}(\bx)$ given by \eqref{eq:defhaty}.

Two error metrics can be considered. A stronger requirement is to learn accurately (up to a permutation) the weights. More formally, we require that the \emph{estimation error} defined below is small:
\begin{equation}\label{eq:quantsmall1}
\min_{\pi}\sum_{i=1}^r\norm{\bw_i-\hat{\bw}_{\pi(i)}}^2,
\end{equation}
where the minimization is with respect to all permutations $\pi:[r]\to[r]$. If we assume that the vectors $\{\bw_i\}_{1\le i \le n}$ and $\{\hat{\bw}_i\}_{1\le i \le n}$ have unit norm, then the quantity in \eqref{eq:quantsmall1} is small if and only if the following quantity is large:
\begin{equation}\label{eq:quantsmall2}
\max_{\pi}\sum_{i=1}^r \left \langle\bw_i, \hat{\bw}_{\pi(i)}\right \rangle.
\end{equation}
A weaker requirement is to predict accurately the output of the network. More formally, we require that the \emph{generalization error} defined below is small:
\begin{equation}\label{eq:quantsmall3}
{\mathbb E} \left\{|y(\bx)-\hat{y}(\bx)|^2\right\},
\end{equation}
where the expectation is with respect to the distribution of $\bx$. Our results of Section \ref{sec:gen} prove that these two requirements are equivalent when $\bx$ is Gaussian: if the stronger requirement does not hold, i.e., the correlation \eqref{eq:quantsmall2} is small, then also the weaker requirement does not hold, i.e., the generalization error \eqref{eq:quantsmall3} is large.

%----------------------------------------------------
\subsection{Tensor Decomposition} \label{subsec:tensor}
%----------------------------------------------------

Tensors are arrays of numbers indicized by multiple integers and they can be regarded as a generalization of matrices (indicized by two integers) and vectors (indicized by a single integer). Similarly to the problem of learning a neural network, many problems involving tensors (e.g., the computation of the rank or the spectral norm) are NP-hard in the worst case \cite{haastad1990tensor, hillar2013most}. However, recent work has focused on the development of provably efficient algorithms, especially for low-rank tensor decompositions, by making suitable assumptions about the input and allowing for approximations \cite{pmlr-v40-Anandkumar15, AGJ17, Ge2015DecomposingO3, hopkins2015tensor, hopkins2016fast, barak2015dictionary, ma2016polynomial, schramm2017fast}. 

The typical setting for the problem of tensor decomposition is as follows. Let $\bw_1, \ldots, \bw_r\in \mathbb R^d$ be vectors of unit norm and, for $k\ge 3$, define the $k$-th order tensor $\bT^{(k)}$ as 
\begin{equation}\label{eq:deftns}
\bT^{(k)} = \sum_{i=1}^r \bw_i^{\otimes k}.
\end{equation}
Given a subset of tensors $\{\bT^{(k)}\}_{3\le k \le \ell}$, the objective is to recover the vectors $\bw_1, \ldots, \bw_r$.

A classical algorithm based on matrix diagonalization \cite{harshman1970foundations, de1996blind} solves the tensor decomposition problem when $\bw_1, \ldots, \bw_r$ are linearly independent and $\ell \ge 3$. The requirement that $\bw_1, \ldots, \bw_r$ are linearly independent immediately implies that $r\le d$. Recent works have focused on the overcomplete case, in which $r > d$. The best algorithms are based on (or match the guarantees of) the SoS hierarchy and these results are reviewed below. 

\noindent {\bf Random vectors.} Assume that $\bw_1, \ldots, \bw_r$ are chosen independently at random from the unit sphere in $\mathbb R^d$. Then, with high probability, tensor decomposition can be solved given $\bT^{(3)}$ and $r$ as large as $d^{3/2}$ (up to logarithmic factors), see Theorem 1.2 in \cite{ma2016polynomial}.

\noindent {\bf Separated unit vectors.} Assume that $\bw_1, \ldots, \bw_r$ have at most $\delta$-correlation, i.e., for any $i, j\in [r]$ with $i\neq j$, $|\left \langle \bw_i,\bw_j\right \rangle |\le \delta$. Then, tensor decomposition can be solved given the tensors of order up to $\log r/\log(1/\delta)$ \cite{schramm2015low}.

\noindent {\bf General unit vectors.} In this scenario, $\bw_1, \ldots, \bw_r$ can be any vectors in $\mathbb R^d$. Then, tensor decomposition can be approximated given the tensors of order up to ${\rm poly}(1/\varepsilon)$, where $\varepsilon$ denotes the Hausdorff distance\footnote{The Hausdorff distance between two finite sets $A$ and $B$ is equal to the maximum between $\max_{a\in A}\min_{b\in B}\norm{a-b}$ and $\max_{b\in B}\min_{a\in A}\norm{a-b}$.} between the original set of weights and the set of estimates, see Theorem 1.6 in \cite{ma2016polynomial}.

%----------------------------------------------------
\section{Lower Bounds on Generalization Error} \label{sec:gen}
%----------------------------------------------------

In our results, we consider a more general predictor $\hat{y}(\bx)$ given by
\begin{equation}\label{eq:defhatyR}
\hat{y}(\bx) = \sum_{i=1}^R \sigma(\left \langle \bx, \hat{\bw}_i\right \rangle),
\end{equation}
i.e., we allow the number $R$ of estimated weights to be different from the number $r$ of unknown weights. Our first theorem holds when the weights $\{\bw_i\}_{1\le i \le r}$ are separated and isotropic, and our second theorem when the weights $\{\bw_i\}_{1\le i \le r}$ are random. 

%----------------------------------------------------
\subsection{Separated Isotropic Weights} \label{subsec:iso}
%----------------------------------------------------

We make the following assumptions on the weights $\{\bw_i\}_{1\le i \le r}$. 

\begin{itemize}

\item[\sf (A1)] \emph{Unit norm}:
\begin{equation}
\norm{\bw_i}=1, \qquad \forall\, i\in [r].
\end{equation}

\item[\sf (A2)] \emph{At most $\delta$-correlation}:
\begin{equation}
|\left \langle \bw_i, \bw_j\right \rangle | \le \delta,\qquad \forall\, i, j \in [r], \mbox{ with }i\neq j.
\end{equation}

\item[\sf (A3)] \emph{Mean $\eta_{\rm avg}$-close to zero}:
\begin{equation}
\norm{\sum_{i=1}^r \bw_i}^2 \le \eta_{\rm avg}\cdot r.
\end{equation}

\item[\sf (A4)] \emph{Covariance $\eta_{\rm var}$-close to scaled identity}: 
\begin{equation}
\norm{\sum_{i=1}^r \bw_i \bw_i^{\sT} -\frac{r}{d} \bI_d}_{\rm op}\le \eta_{\rm var}\cdot r/d.
\end{equation}

\end{itemize}

%In the results concerning the generalization error (the lower bound of Section \ref{subsec:gen} and the reduction to tensor decomposition of Section \ref{subsec:genhard}), the weights $\{\bw_i\}_{1\le i \le r}$ are required to satisfy all the assumptions above. In the results of Section \ref{subsec:learn} concerning the problem of learning the parameters of a neural network, the weights $\{\bw_i\}_{1\le i \le r}$ are required to satisfy  (A1) (and also (A2) for the case with noise). The assumptions (A3) and (A4) mean that the weights are roughly isotropic. The factor $r/d$ in front of $\bI_d$ in the RHS of \eqref{eq:idcov} is necessary in order to make (A1) and (A4) compatible.

It is simple to produce weight vectors that satisfy these assumptions. If the matrix of the weights is equal to the identity matrix, then the assumptions hold with $\delta = 0$, $\eta_{\rm avg} =1$, and $\eta_{\rm var} =0$. If we center and rescale the weights by a factor $\sqrt{d/(d-1)}$, we have that the assumptions hold with $\delta= \frac{d+1}{d(d-1)}\approx 1/d$, $\eta_{\rm avg}=0$, and $\eta_{\rm var}=2$. For $r=d+1$, we can take $\widetilde{\bW}$ to be Haar distributed conditional on $\widetilde{\bW}^{\sT}{\boldsymbol 1_{d+1}} = \b0_d$, and let $\{\bw_i\}_{1\le i \le r}$ be $\sqrt{(d+1)/d}$ times the rows of $\widetilde{\bW}$
(these are just the rotations of the vertices of the standard simplex).
 Then, the assumptions hold with $\delta = 1/d$ and $\eta_{\rm avg}=\eta_{\rm var}=0$. 
For $r>d+1$, we concatenate $r/(d+1)$ of these matrices. By doing so, we still have that $\eta_{\rm avg}=\eta_{\rm var}=0$. We expect $\delta$  to be small (say of order $1/\sqrt{d}$).

The result below, whose proof is contained in Appendix \ref{app:mainproof}, considers the case of a Gaussian input distribution and rules out a scenario in which the weights are not estimated well, but the generalization error is still small.

\begin{theorem}[Lower Bound on Generalization Error for Separated Isotropic Weights]\label{th:lowergen}
Consider a two-layer neural network with input dimension $d$, $r$ hidden nodes, and activation function $\sigma\in L^2(\mathbb R, e^{-x^2/2})$. Assume that the weights $\{\bw_i\}_{1\le i\le r}$ satisfy the assumptions {\sf (A1)}-{\sf (A4)} for positive $\delta, \eta_{\rm avg}$ and $\eta_{\rm var}$ such that $1-\delta\cdot(1+\eta_{\rm var})\cdot r/d\ge 0$. Let $y(\bx)$ and $\hat{y}(\bx)$ be defined in \eqref{eq:defy} and \eqref{eq:defhatyR}. Assume that the estimated weights $\{\hat{\bw}_i\}_{1\le i\le R}$ satisfy the assumption {\sf (A1)} and have at most $\epsilon$-correlation with the ground-truth weights $\{\bw_i\}_{1\le i\le r}$, i.e., for some $\epsilon >0$, $|\left \langle \bw_i,\hat{\bw}_j\right \rangle |\le \epsilon$, for all $i\in [r]$ and $j\in [R]$. Then, the following lower bound on the generalization error holds:
\begin{equation}\label{eq:lbgen}
{\mathbb E}\big\{|y(\bx)-\hat{y}(\bx)|^2\big\}\ge \left(\min_{a, b\in \mathbb R}{\mathbb E}\left\{\Big|y(\bx)-\big(a+b\norm{\bx}^2\big)\Big|^2\right\}-c_1\right)\left(1-c_2\right),
\end{equation}
where the expectation is with respect to $\bx\sim \normal(\b0_d,\bI_d)$ and the terms $c_1$ and $c_2$ are given by 
\begin{equation}\label{eq:defc1c2}
c_1 = 2\hat{\sigma}_1^2 \cdot \eta_{\rm avg}\cdot r +2\hat{\sigma}_2^2\cdot \eta_{\rm var}^2\cdot r^2/d,\qquad c_2 = \frac{2\epsilon\cdot(1+\eta_{\rm var})\cdot R/d}{1-\delta\cdot(1+\eta_{\rm var})\cdot r/d},
\end{equation}
with $\hat{\sigma}_1$ and $\hat{\sigma}_2$ defined in \eqref{eq:her12}. 

If we also assume that $\sigma$ is even, then \eqref{eq:lbgen} holds with $c_1$ and $c_2$ given by 
\begin{equation}\label{eq:defc1c2_extra}
c_1 = 2\hat{\sigma}_2^2\cdot \eta_{\rm var}^2\cdot r^2/d,\qquad c_2 = \frac{2\epsilon^2\cdot(1+\eta_{\rm var})\cdot R/d}{1-\delta^2\cdot(1+\eta_{\rm var})\cdot r/d}.
\end{equation}
\end{theorem}

Some remarks are of order.

\begin{itemize}

\item Note that the generalization error
\begin{equation}
\min_{a, b\in \mathbb R}{\mathbb E}\left\{\big|y(\bx)-\big(a+b\norm{\bx}^2\big)\big|^2\right\}
\end{equation}
is that of a trivial predictor having access only to the norm of the input. Hence, if the weights are not estimated well, then the generalization error is close to that of a predictor that does not really use the input.

\item The assumption that the weights $\{\bw_i\}_{1\le i\le r}$ and $\{\hat{\bw}_i\}_{1\le i\le R}$ have unit norm mainly serves to simplify the proof. On the contrary, the assumption that the weights $\{\bw_i\}_{1\le i\le r}$ are roughly isotropic is crucial. Indeed, if either {\sf (A3)} or {\sf (A4)} do not hold, then it might be possible to learn the mean vector or the covariance matrix of the weights, which could reduce the generalization error for activation functions that have a non-zero linear or quadratic component. Indeed, consider the following example: 
$\sigma(x) = x$,  $\{\bw_i\}_{i\le r}$ arbitrary, and $\hat{\bw}_i =\overline{\bw}\equiv \sum_{i=1}^r\bw_i/r$ for all $i$. Clearly, the weights are not estimated correctly. 
However, the generalization error is $0$ for any input $\bx\in \mathbb R^d$ (and is superior to the one of the trivial predictor). 

\item Let us evaluate the bound for some natural choices of the weights $\{\bw_i\}_{1\le i \le r}$.	Recall that, if $\sigma$ is even (odd), then $\hat{\sigma}_k=0$ for $k$ odd (even). If the matrix of the weights is equal to the identity matrix and $\sigma$ is even, then the generalization error of the neural network is close to that of a trivial predictor, namely, the neural network does not generalize well, as long as $\epsilon^2\cdot R/d$ is small. Suppose now that we center and rescale the weights and that we pick $\sigma$ odd. Then, the neural network does not generalize well as long as $\epsilon\cdot R/d$ is small. If the weights are the rescaled rows of $r/(d+1)$ matrices $\widetilde{\bW}$, where $\widetilde{\bW}$ is Haar distributed conditional on $\widetilde{\bW}^{\sT}{\boldsymbol 1_{d+1}} = \b0_d$, then, for any $\sigma$, the neural network does not generalize well as long as $\epsilon\cdot R/d$ and $\delta\cdot r/d$ are small. Furthermore, when $\sigma$ is even, we only require that $\epsilon^2 \cdot R/d$ and $\delta^2\cdot r/d$ are small.

\end{itemize}

%----------------------------------------------------
\subsection{Random Weights} \label{sub:iso}
%----------------------------------------------------

We assume that the weights $\{\bw_i\}_{1\le i \le r}$ have the following form:
\begin{equation}\label{eq:defrandomw}
\bw_i =\frac{\bg_i -\frac{1}{r}\sum_{j=1}^r \bg_j}{\norm{\bg_i -\frac{1}{r}\sum_{j=1}^r \bg_j}}, \qquad \forall\,i\in [r],
\end{equation}
where $\{\bg_i\}_{1\le i \le r}\sim_{\rm i.i.d.}\normal (\b0_d, \bI_d/d)$. The result below, whose proof is contained in Appendix \ref{app:mainproof_G}, is similar in spirit to Theorem \ref{th:lowergen} and it applies to a setting with random weights. 

\begin{theorem}[Lower Bound on Generalization Error for Random Weights]\label{th:lowergen_G}
Consider a two-layer neural network with input dimension $d$, $r$ hidden nodes, and activation function $\sigma\in L^2(\mathbb R, e^{-x^2/2})$ such that $\hat{\sigma}_2=0$, where $\hat{\sigma}_2$ is defined in \eqref{eq:her12}. Assume that the weights $\{\bw_i\}_{1\le i\le r}$ have the form \eqref{eq:defrandomw}. 
Let $y(\bx)$ and $\hat{y}(\bx)$ be defined in \eqref{eq:defy} and \eqref{eq:defhatyR}. For some $\epsilon \in (0, 1)$, define
\begin{equation}\label{eq:defShat}
\hat{\mathcal S}_\epsilon = \{\{\hat{\bw}_i\}_{1\le i\le R} : \norm{\hat{\bw}_i}=1 \,\,\,\forall\, i\in [R], |\left \langle \bw_i, \hat{\bw}_j\right \rangle |\le \epsilon \,\,\,\forall\,i\in [r]\,\,\,\forall\,j\in [R]\}.
\end{equation}
As $r,d\to\infty$, assume that
\begin{equation}\label{eq:ass_G}
\epsilon=o(1),\qquad r= o(d^2/(\log d)^2).
\end{equation}
Then, for a sequence of vanishing constants $\eta(r, d) = o(1)$, with high probability with respect to $\bw=(\bw_i)_{i\le r}$, 
\begin{equation}\label{eq:bdnew_G}
\sup_{\{\hat{\bw}_i\}_{1\le i\le R} \in \hat{\mathcal S}_\epsilon}{\mathbb E}\big\{|y(\bx)-\hat{y}(\bx)|^2\big\}\ge \left(\Var\left\{y(\bx)\right\}-r\cdot \eta(r, d)\right)\left(1-\frac{R}{r}\cdot \eta(r, d)\right),
\end{equation}
where the expectation and the variance is with respect to $\bx\sim \normal(\b0_d,\bI_d)$.
\end{theorem}	

Some remarks are of order.
\begin{itemize}

\item Note that $\left\langle \bx, \bw_i\right\rangle$ is of order $1$, hence the term $\Var\left\{y(\bx)\right\}$ is of order $r$. Consequently, in the limit $r,d\to\infty$, the term $r\cdot \eta(r, d)$ is negligible compared to $\Var\left\{y(\bx)\right\}$. 

\item The hypothesis that $\hat{\sigma}_2=0$ can be removed at the cost of a less tight lower bound. For general $\sigma$, we have that 
\begin{equation}
\begin{split}
\sup_{\{\hat{\bw}_i\}_{1\le i\le R} \in \hat{\mathcal S}_\epsilon}{\mathbb E}\big\{|y(\bx)-\hat{y}(\bx)|^2\big\}&\ge \left(\min_{\substack{a\in \mathbb R\\\bA\in \mathbb R^{d\times d}}}{\mathbb E}\left\{\left|y(\bx)-\left(a+\<\bx, \bA\bx\>\right)\right|^2\right\}-r\cdot \eta(r, d)\right)\\
&\hspace{2em}\cdot \left(1-\frac{R}{r}\cdot \eta(r, d)\right).
\end{split}
\end{equation}
In fact, note that $\Var\left\{y(\bx)\right\} = \min_{a\in \mathbb R}\mathbb E\left\{|y(\bx)-a|^2\right\}$. 

\item Theorem \ref{th:lowergen_G} covers regimes different from those of Theorem \ref{th:lowergen}. Indeed, the result of this section guarantees that the generalization error of the neural network is close to that of a trivial predictor for any $\sigma$ such that $\hat{\sigma}_2=0$ and for $r$ up to $d^2$ (modulo logarithmic factors),
unless the weights are estimated `better than random', namely with a non-vanishing correlation. We also allow predictors with a number of nodes $R$ that can be larger than the number of nodes $r$ of the original neural network, as long as $R$ and $r$ are of the same order.

%\item The assumption $1/(\epsilon\sqrt{r}) = o(1)$ is natural, since we think of $r$ as being of order $d$ and $\epsilon$ at least of order $1/\sqrt{d}$.

\end{itemize}

The key technical step in the proof is upper bounding the third-order correlation 
\begin{equation}\label{eq:3ord}
\frac{1}{R}\sum_{i\le r, j\le R}\<\bw_i,\hat{\bw}_j\>^3
\end{equation}
 uniformly over all estimates
such that $\max_{i,j}|\<\bw_i,\hat{\bw}_j\>|\le \eps$. A naive bound would be $r \eps^3$, while using the approximate isotropicity of the $\bw_i$ yields an upper bound of order $\eps \max(1,r/d)$. For $\eps \approx 1/\sqrt{d}$ this would vanish only in the regime $r\ll d^{3/2}$. In order to obtain a non-trivial result for $r\gg d^{3/2}$,
we use the randomness of the $\bw_i$, together with an epsilon-net argument and several ad-hoc estimates, which eventually yields that the quantity in \eqref{eq:3ord} is $o(1)$ under the stated assumptions.

%----------------------------------------------------
\section{Learning a Neural Network and Tensor Decomposition} \label{sec:learn}
%----------------------------------------------------

We now present reductions from tensor decomposition to the problem of learning the weights of a two-layer neural network. No assumption on the input distribution is necessary and the results hold for any set of inputs. Before giving the statement, let us formally define what we mean when we say that it is algorithmically hard to learn the weights $\{\bw_i\}_{1\le i\le r}$.

\begin{definition}[$\epsilon$-Hardness of Learning]\label{def:hardeps}
A weight-learning problem is defined by a triple $(\eps,\cS,f)$, where $\epsilon\in (0, 1)$, $\cS$ is a set of possible weights
\begin{equation}
\mathcal S\subseteq\{\{\bw_i\}_{1\le i\le r} : \norm{\bw_i}=1, \,  \forall\, i\in [r]\},
\end{equation}
and $f: \mathcal S\to\mathcal I$ is a function, where $\mathcal I$ denotes a set of inputs. We always assume that $r$ and the size of  $\mathcal I$ are bounded by polynomials in $d$. 

We say that the problem $(\eps,\cS,f)$ is \emph{hard} (or, the problem is \emph{$\epsilon$-hard}) if there is no algorithm $\mathcal A$ that, given as input $f(\bw_1, \ldots, \bw_r)$, fulfills the following two properties:
\begin{itemize}

\item[\sf (P1)] $\mathcal A$ outputs $\{\hat{\bw}_i\}_{1\le i\le R}$ of unit norm such that, for some $i\in [r]$ and $j\in [R]$, $|\left \langle \bw_i, \hat{\bw}_j\right \rangle| \ge \epsilon$;

\item[\sf (P2)] $\mathcal A$ has complexity which is polynomial in $d$.

\end{itemize}

\end{definition}

The result below, whose proof is contained in Appendix \ref{sec:proofhard}, provides a reduction for activation functions that are polynomials whose degree is \emph{at most} the order of the tensor to be decomposed.

\begin{theorem}[Learning a Neural Network and Tensor Decomposition]\label{th:comp}
Fix an integer $\ell\ge 3$ and let $y(\bx)$ be defined in \eqref{eq:defy}, where $\sigma$ is the activation function. For $\bx_1, \ldots, \bx_n\in \mathbb R^d$, let $\mathcal P(\bx_1, \ldots, \bx_n)$ be the problem of learning $\{\bw_i\}_{1\le i\le r}$ given as input $\{\bx_j\}_{1\le j \le n}$ and $\{y(\bx_j)\}_{1\le j \le n}$. Then, the following results hold.

\begin{enumerate}
\item Assume that, given as input the tensor $\bT^{(\ell)}$ defined in \eqref{eq:deftns}, the problem of learning $\{\bw_i\}_{1\le i\le r}\in \mathcal S$ is $\epsilon$-hard in the sense of Definition \ref{def:hardeps} for some $\epsilon>0$. Let the activation function $\sigma$ be a polynomial with ${\rm deg}(\sigma)\le \ell$ and ${\rm par}(\sigma)={\rm par}(\ell)$. Then, for any $\bx_1, \ldots, \bx_n\in \mathbb R^d$, the problem $\mathcal P(\bx_1, \ldots, \bx_n)$ is $\epsilon$-hard in the sense of Definition \ref{def:hardeps}.

\item Assume that, given as input the tensors $\bT^{(\ell)}$ and $\bT^{(\ell+1)}$ defined in \eqref{eq:deftns}, the problem of learning $\{\bw_i\}_{1\le i\le r}\in \mathcal S$ is $\epsilon$-hard in the sense of Definition \ref{def:hardeps} for some $\epsilon>0$. Let the activation function $\sigma$ be a polynomial with ${\rm deg}(\sigma)\le \ell+1$. Then, for any $\bx_1, \ldots, \bx_n\in \mathbb R^d$, the problem $\mathcal P(\bx_1, \ldots, \bx_n)$ is $\epsilon$-hard in the sense of Definition \ref{def:hardeps}.
\end{enumerate}

\end{theorem}

In words, learning a two-layer neural network whose activation function is a polynomial of degree $\ell$ and assigned parity (i.e., either even or odd) is as hard as solving tensor decomposition given the tensor of order $\ell$ with the same parity. Furthermore, learning a two-layer neural network whose activation function is a polynomial of degree $\ell+1$ (without any assumption on its parity) is as hard as solving tensor decomposition given the tensors of order $\ell$ and $\ell+1$. In Appendix \ref{subsec:genhard}, we consider a slightly different model of two-layer neural network with an additive error term. By doing so, we can prove a reduction with activation functions that are polynomials with degree \emph{larger} than the order of the tensor.

%----------------------------------------------------
\section{Generalization Error and Tensor Decomposition} \label{sec:disc}
%----------------------------------------------------

%By combining the results of Section \ref{sec:gen} and \ref{sec:learn}, we obtain reductions from the tensor decomposition problem to the problem of finding a predictor of a two-layer neural network with small generalization error. 

We now present reductions from tensor decomposition to the problem of finding a predictor of a two-layer neural network with small generalization error. Similarly to Section \ref{sec:learn}, no assumption is necessary on the distribution of the samples given as input to the learning algorithm. However, when taking the expectation to compute the generalization error, we assume that $\bx\sim \normal(\b0, \bI_d)$. 

%For positive $\delta, \eta_{\rm avg}$ and $\eta_{\rm var}$ such that $1-\delta\cdot(1+\eta_{\rm var})\cdot r/d\ge 0$, let
%\begin{equation}
%\mathcal S'\subseteq\{\{\bw_i\}_{1\le i\le r} : \mbox{assumptions {\sf (A1)}-{\sf (A4)} hold}\}.
%\end{equation}
%We now formally define what we mean when we say that it is algorithmically hard to find a predictor $\hat{y}(\bx)$ of the output $y(\bx)$ of a two-layer neural network.

%\begin{definition}[$(c_1, c_2)$-Algorithmic Hardness of Generalization]\label{def:hardgen}
%Fix $c_1, c_2>0$ and consider a two-layer neural network with input $\bx\in \mathbb R^d$ and output $y(\bx)$ defined in \eqref{eq:defy} for some activation function $\sigma$ and weights $\{\bw_i\}_{1\le i\le r}\in \mathcal S'$.  We say that the neural network $y(\bx)$ has \emph{$(c_1, c_2)$-bad generalization error} if, for any $\bx_1, \ldots, \bx_n\in \mathbb R^d$ and for any algorithm that, given as input $\{\bx_j\}_{1\le j\le n}$ and $\{y(\bx_j)\}_{1\le j\le n}$, outputs estimates $\{\hat{\bw}_i\}_{1\le i\le r}$ of unit norm with polynomial complexity, we have that \eqref{eq:lbgen} holds, where $\hat{y}(\bx)$ is defined in \eqref{eq:defhaty} and the expectation is with respect to the distribution of $\bx$. 
%\end{definition}

The corollary below considers the case of separated and isotropic weights and  its proof is readily obtained by combining the results of Theorem \ref{th:lowergen} and \ref{th:comp}. 

\begin{corollary}[Generalization Error and Tensor Decomposition for Separated Isotropic Weights]\label{cor:hardgen}
Fix an integer $\ell\ge 3$, and, for positive $\delta, \eta_{\rm avg}$ and $\eta_{\rm var}$ such that $1-\delta\cdot(1+\eta_{\rm var})\cdot r/d\ge 0$, let
\begin{equation}
\mathcal S'\subseteq\{\{\bw_i\}_{1\le i\le r} : \mbox{assumptions {\sf (A1)}-{\sf (A4)} hold}\}.
\end{equation}
We have the following results.

\begin{enumerate}

\item Assume that, given the tensor $\bT^{(\ell)}$ defined in \eqref{eq:deftns}, the problem of learning $\{\bw_i\}_{1\le i\le r}\in\mathcal S'$ is $\epsilon$-hard in the sense of Definition \ref{def:hardeps} for some $\epsilon>0$. Let $y(\bx)$ be defined in \eqref{eq:defy}, where $\sigma$ is a polynomial with ${\rm deg}(\sigma)\le \ell$ and ${\rm par}(\sigma)={\rm par}(\ell)$. Then, for any $\bx_1, \ldots, \bx_n\in \mathbb R^d$ and for any polynomial algorithm that, given as input $\{\bx_j\}_{1\le j \le n}$ and $\{y(\bx_j)\}_{1\le j \le n}$, outputs $\{\hat{\bw}_i\}_{1\le i\le R}$ of unit norm, we have that
\begin{equation}\label{eq:lbgen_2}
{\mathbb E}\big\{|y(\bx)-\hat{y}(\bx)|^2\big\}\ge \left(\min_{a, b\in \mathbb R}{\mathbb E}\left\{\Big|y(\bx)-\big(a+b\norm{\bx}^2\big)\Big|^2\right\}-c_1\right)\left(1-c_2\right),
\end{equation}
where $\hat{y}(\bx)$ is defined in \eqref{eq:defhatyR}, the expectation is with respect to $\bx\sim \normal(\b0_d,\bI_d)$ and the terms $c_1$ and $c_2$ are given by
\begin{equation}\label{eq:defc1c2_2}
c_1 = 2\hat{\sigma}_1^2 \cdot \eta_{\rm avg}\cdot r +2\hat{\sigma}_2^2\cdot \eta_{\rm var}^2\cdot r^2/d,\qquad c_2 = \frac{2\epsilon\cdot(1+\eta_{\rm var})\cdot R/d}{1-\delta\cdot(1+\eta_{\rm var})\cdot r/d},
\end{equation}
with $\hat{\sigma}_1$ and $\hat{\sigma}_2$ defined in \eqref{eq:her12}. If we also assume that $\sigma$ is even, then \eqref{eq:lbgen} holds with $c_1$ and $c_2$ given by 
\begin{equation}\label{eq:defc1c2_extra_2}
c_1 = 2\hat{\sigma}_2^2\cdot \eta_{\rm var}^2\cdot r^2/d,\qquad c_2 = \frac{2\epsilon^2\cdot(1+\eta_{\rm var})\cdot R/d}{1-\delta^2\cdot(1+\eta_{\rm var})\cdot r/d}.
\end{equation}

\item Assume that, given the tensors $\bT^{(\ell)}$ and $\bT^{(\ell+1)}$ defined in \eqref{eq:deftns}, the problem of learning $\{\bw_i\}_{1\le i\le r}\in\mathcal S'$ is $\epsilon$-hard in the sense of Definition \ref{def:hardeps} for some $\epsilon > 0$. Let $y(\bx)$ be defined in \eqref{eq:defy}, where $\sigma$ is a polynomial with ${\rm deg}(\sigma)\le \ell+1$. Then, for any $\bx_1, \ldots, \bx_n\in \mathbb R^d$ and for any polynomial algorithm that, given as input $\{\bx_j\}_{1\le j \le n}$ and $\{y(\bx_j)\}_{1\le j \le n}$, outputs $\{\hat{\bw}_i\}_{1\le i\le R}$ of unit norm, we have that \eqref{eq:lbgen_2} holds, where $\hat{y}(\bx)$ is defined in \eqref{eq:defhatyR}, the expectation is with respect to $\bx\sim \normal(\b0_d,\bI_d)$ and the terms $c_1$ and $c_2$ are given by \eqref{eq:defc1c2_2}. Furthermore, if $\sigma$ is even, then the terms $c_1$ and $c_2$ are given by \eqref{eq:defc1c2_extra_2}.
\end{enumerate}

\end{corollary}

A reduction for the case of random weights is contained in Theorem \ref{th:hardgen_Gnew}, stated in Section \ref{sec:intro}. Its proof follows by combining the result of Theorem \ref{th:lowergen_G} with $R=r$ with the same proof of Theorem \ref{th:comp}.   

Let us now summarize briefly some implications of our results. The discussion in Section \ref{subsec:tensor} suggests that tensor decomposition is hard in the following cases: if the weights are random vectors, given $\bT^{(3)}$ and for $r\gg d^{3/2}$; if the weights are separated unit vectors, given $\{\bT^{(k)}\}_{3\le k \le \ell}$, for fixed $\ell$ and for $r\gg d$. By setting $R=r$, in our paper we consider a model similar to that of \cite{tian2017symmetry, sedghi2015provable, janzamin2015beating, zhong2017recovery}. Our results suggest that it will be difficult to extend those recovery schemes to several interesting regimes:

\begin{enumerate}

\item $d^{3/2}\ll r\ll d^{2}$ for random weights and activation function $\sigma(x)=a_0 +a_1 x+a_3 x^3$ for some $a_0$, $a_1$, $a_3\in \mathbb R$.

\item $d\ll r\ll d/\epsilon$ for separated isotropic weights and polynomial activation function;

\item $d\ll r\ll d/\epsilon^2$ for separated isotropic weights and even polynomial activation function;

\end{enumerate}

\begin{figure*}[p]
    \centering
    \subfloat[Generalization error, $r=50$.]{\includegraphics[width=.48\columnwidth]{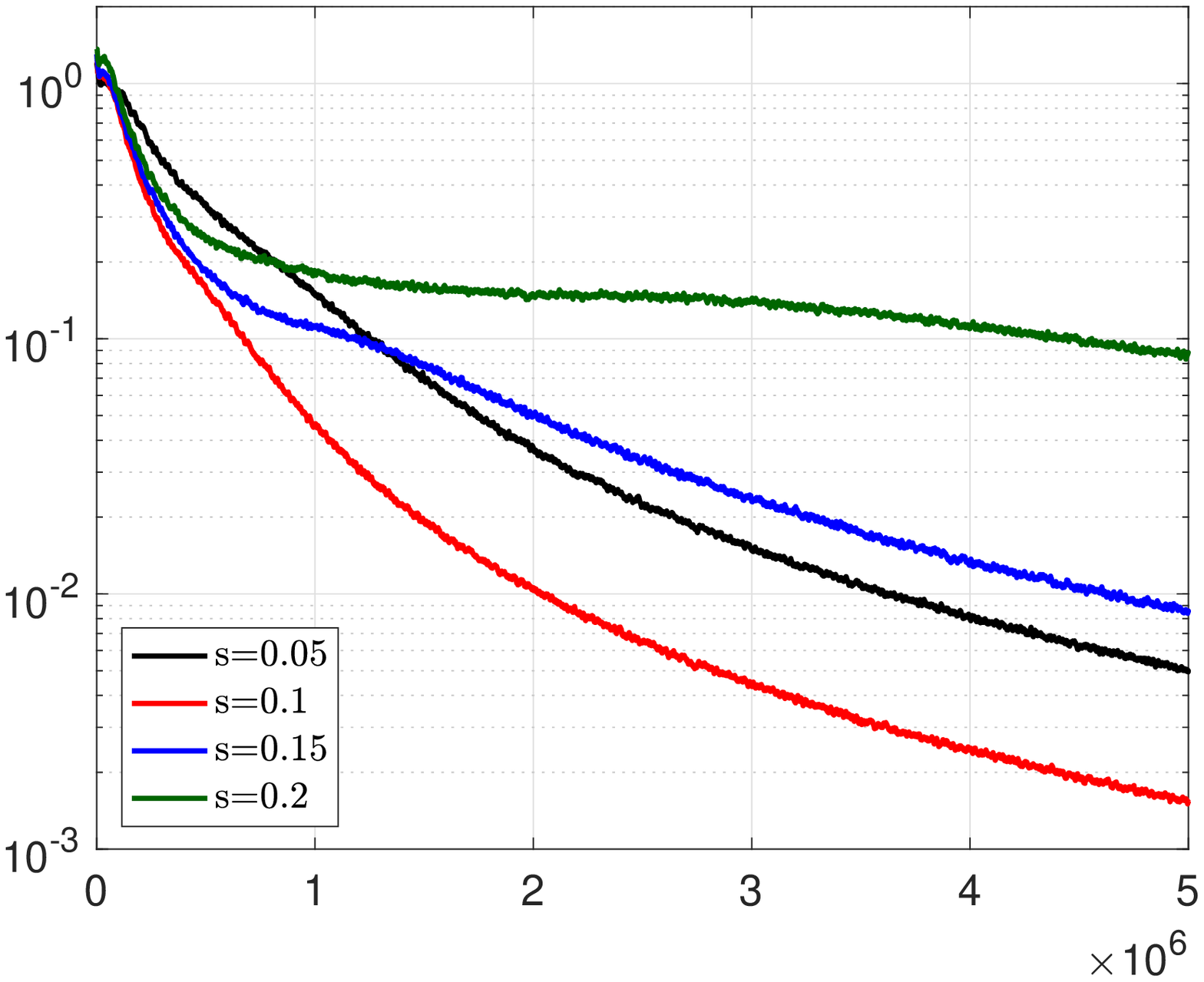}} 
    \subfloat[Weight estimation error, $r=50$.]{\includegraphics[width=.48\columnwidth]{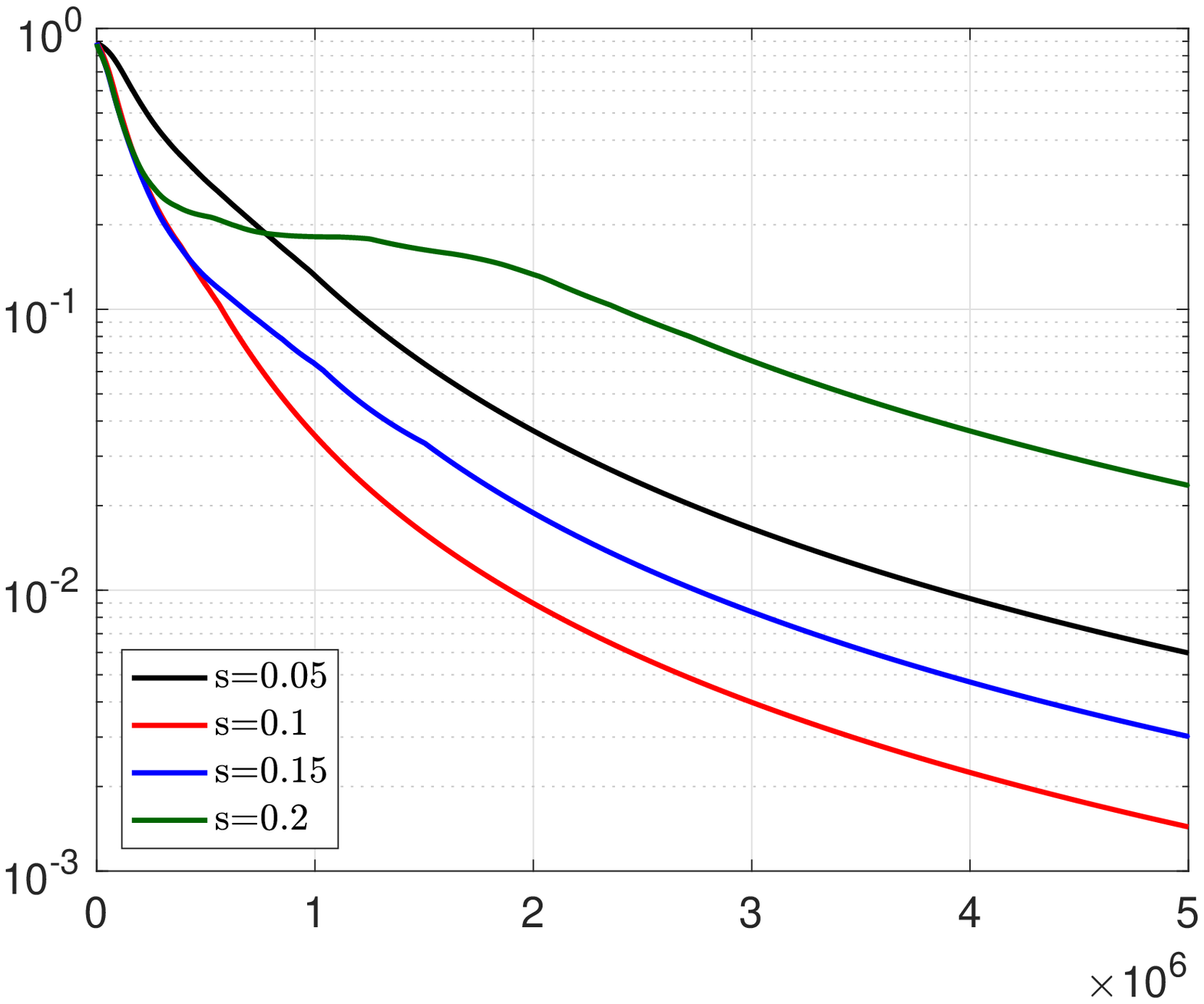}}\\
   \vspace{-.5em}
    \subfloat[Generalization error, $r=350$.]{\includegraphics[width=.48\columnwidth]{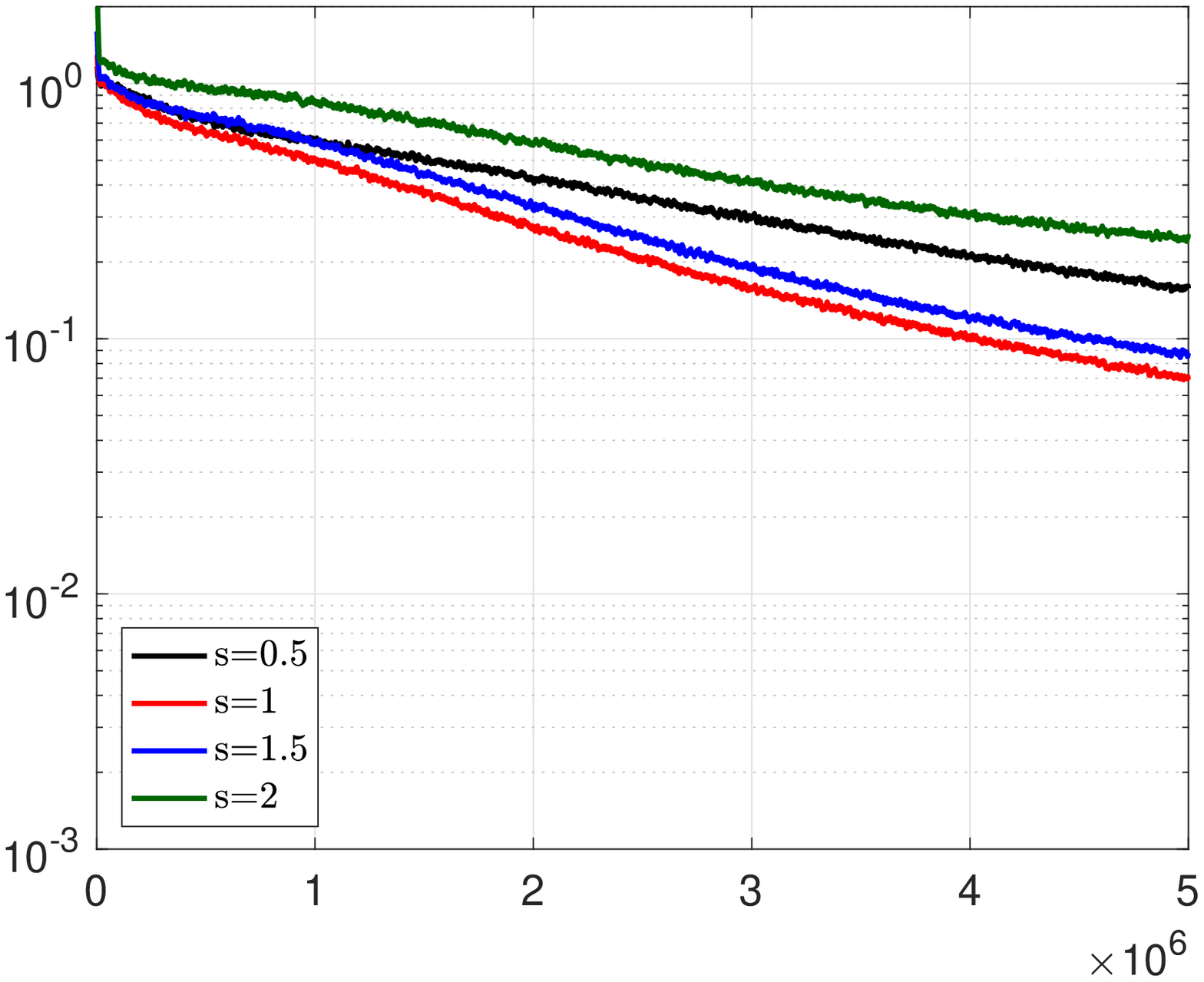}} 
    \subfloat[Weight estimation error, $r=350$.]{\includegraphics[width=.48\columnwidth]{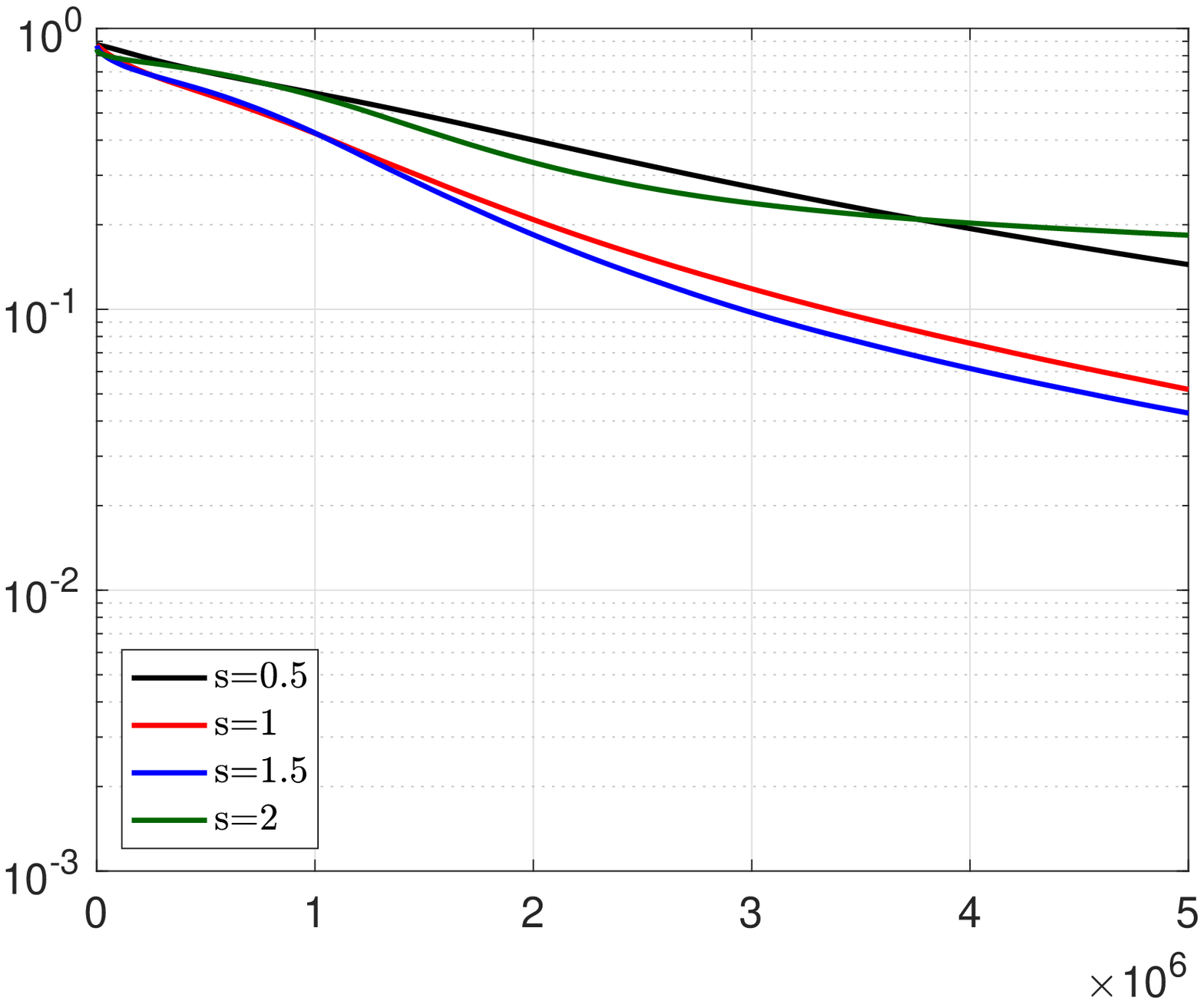}}\\
   \vspace{-.5em}
    \subfloat[Generalization error, $r=2500$.]{\includegraphics[width=.48\columnwidth]{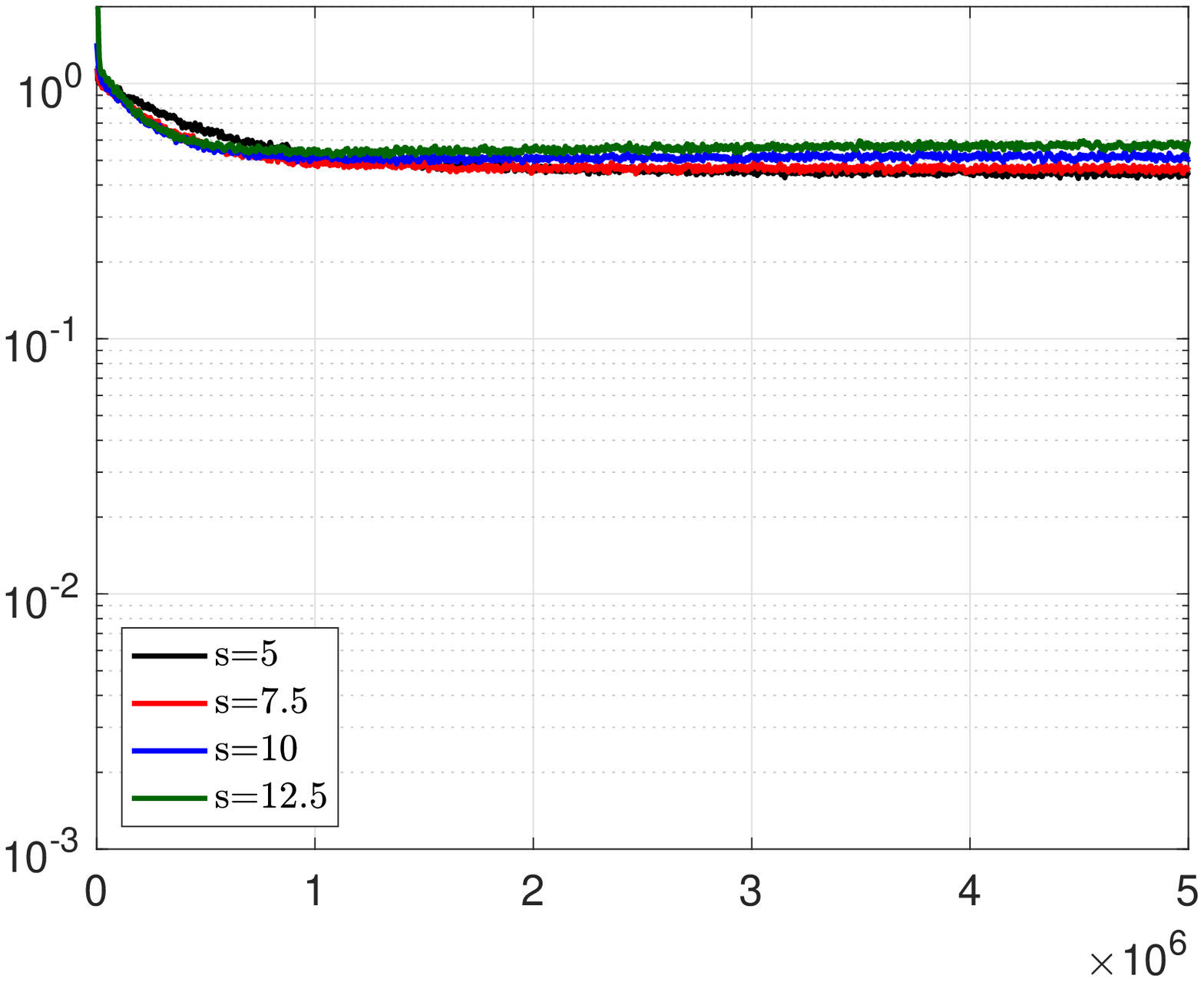}} 
    \subfloat[Weight estimation error, $r=2500$.]{\includegraphics[width=.48\columnwidth]{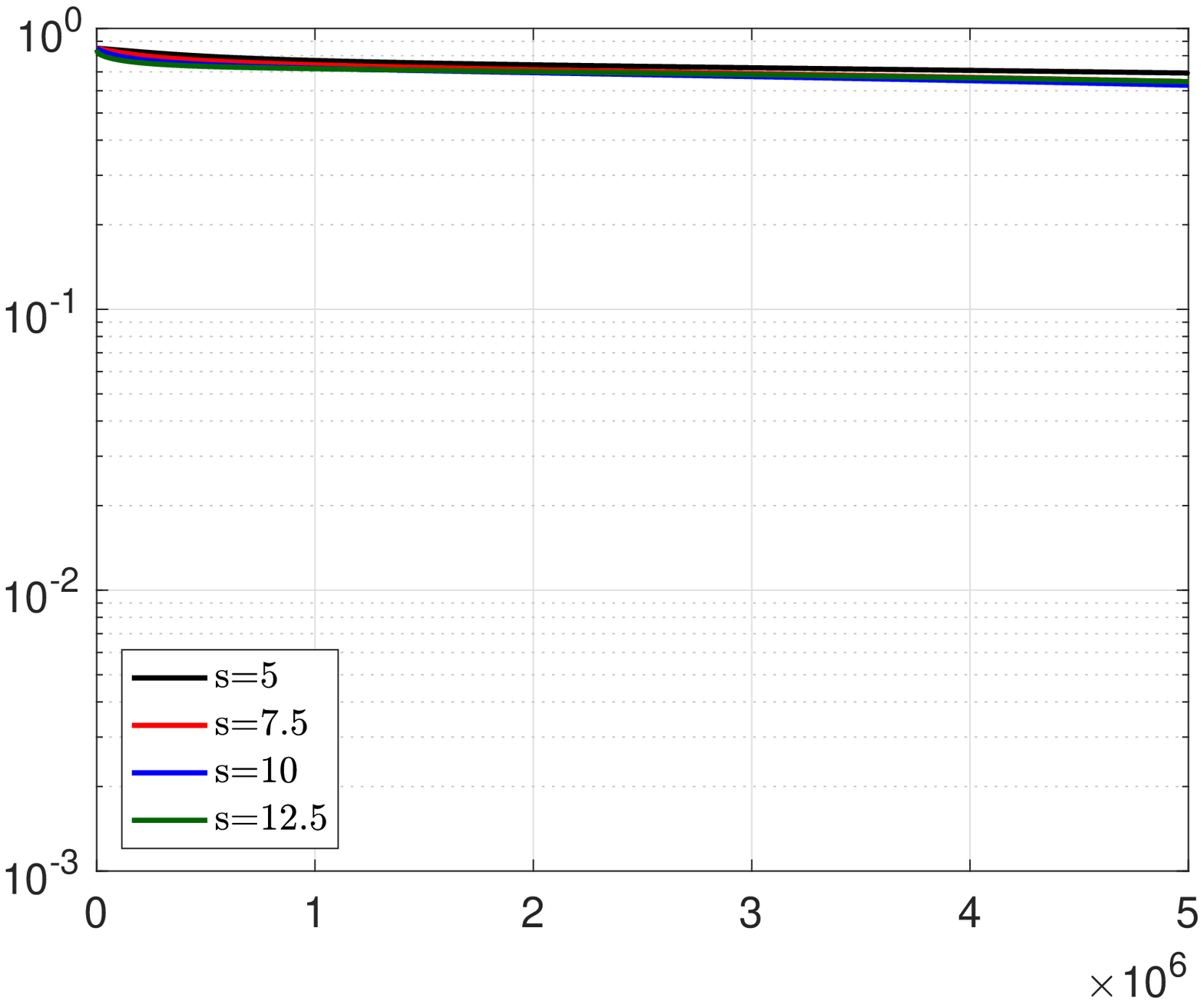}}\\
\caption{Performance of stochastic gradient descent with Gaussian input distribution and separated isotropic weights. }
\label{fig:res}
\end{figure*}

%----------------------------------------------------
\section{Numerical Experiments} \label{sec:simu}
%----------------------------------------------------

The setting for the numerical simulations is described as follows. We consider a two-layer neural network with input dimension $d=50$ and $r$ hidden nodes, with $r\in \{50, 350, 2500\}$. The activation function $\sigma$ is equal to $\tanh\left(5x/2\right)$. The weights $\{\bw_i\}_{1\le i\le r}$ are obtained by concatenating $r/d$ random unitary matrices of size $d\times d$ that are independent and identically distributed according to the Haar measure. In particular, each of these matrices is obtained from the SVD of a matrix whose entries that are $\sim_{\rm i.i.d.}\normal(0, 1)$. Then, the weights are centered by subtracting their empirical mean. We generate $n = 5\cdot 10^6$ samples $\{(\bx_j, y(\bx_j))\}_{1\le j \le n}$, where $\{\bx_j\}_{1\le j \le n}\sim_{\rm i.i.d.} \normal(\b0, \bI_d)$ and $y(\bx)$ is given by \eqref{eq:defy}. We perform $n$ iterations of stochastic gradient descent with a fixed step size $s$. We also perform Polyak-Ruppert averaging, i.e., the algorithm outputs at step $j\in [n]$ the average of the estimates obtained so far. Let $\hat{y}_j(\bx)$ be the predictor given by \eqref{eq:defhaty}, where $\{\hat{\bw}_i\}_{1\le i \le r}$ are the weights outputted by the algorithm at step $j\in [n]$.

The results are presented in Figure \ref{fig:res}, where we plot two different performance metrics. On the left, we have the generalization error
\begin{equation}
\frac{(y(\bx_j)-\hat{y}_j(\bx_j))^2}{(y(\bx_j)-y_{\rm LS}(\bx_j))^2}, \qquad j\in [n],
\end{equation}
where $y_{\rm LS}(\bx_j)$ is the prediction of the least-squares estimator with access only to the norm of the input. In order to obtain a smoother curve, we average the results over a window of size $10^4$.  Note that, for $r=2500$, the estimator $y_{\rm LS}(\bx_j)$ generalizes poorly, in the sense that its loss is close to $\Var\left\{y(\bx)\right\}$. On the right, we have the weight estimation error
\begin{equation}
\frac{1}{2r}\sum_{i=1}^r \min_{j\in [r]} \norm{\hat{\bw}_i-\bw_j}^2 + \frac{1}{2r}\sum_{i=1}^r \min_{j\in [r]} \norm{\hat{\bw}_j-\bw_i}^2,
\end{equation} 
which represents the average of the minimum distances between the ground-truth weights $\{\bw_i\}_{1\le i \le r}$ and the estimated weights $\{\hat{\bw}_i\}_{1\le i \le r}$. Different pairs of plots correspond to different choices for the number of hidden nodes $r$, and in each plot we have several curves for different values of the step $s$. Similar results are obtained by taking random weights of the form \eqref{eq:defrandomw}. 

The numerical results corroborate the picture that we have proved in the paper for Gaussian features. As the number of hidden units $r$ becomes much larger than $d$ (from $r=d$ to $r\approx d^{3/2}$ and $r\approx d^{2}$), the problem of learning the weights of the neural network becomes harder and harder, similarly to what happens for tensor decomposition. Furthermore, the generalization error has the same qualitative behavior of the weight estimation error: the neural network generalizes well if and only if the weights are learned accurately.

%----------------------------------------------------
\section*{Acknowledgement}
%----------------------------------------------------

M.~M. was supported by an Early Postdoc.Mobility fellowship from the Swiss National Science Foundation and by the Simons Institute for the Theory of Computing. A. M. was partially supported by grants NSF DMS-1613091 and NSF CCF-1714305.

\bibliographystyle{amsalpha}
\bibliography{all-bibliography}

\appendix

%----------------------------------------------------
\section{Lower Bound on Generalization Error for Separated Isotropic Weights: Proof of Theorem \ref{th:lowergen}} \label{app:mainproof}
%----------------------------------------------------

Let us start by recalling a basic fact about Hermite polynomials and Fourier analysis on Gaussian spaces. The interested reader is referred to Section 11.2 in \cite{o2014analysis} for further details. Given $\sigma \in L^2(\mathbb R, e^{-x^2/2})$, its Hermite expansion can be written as  
\begin{equation}\label{eq:exp}
\sigma(z) = \sum_{k\in \mathbb N} \hat{\sigma}_k h_k(z),
\end{equation}
with 
\begin{equation}
\hat{\sigma}_k = {\mathbb E}_{G\sim \normal(0, 1)}\left\{h_k(G)\sigma(G)\right\},
\end{equation}
where $\hat{\sigma}_k$ is the $k$-th Hermite coefficient of $\sigma$ and $h_k$ is the $k$-th Hermite polynomial. It is helpful to write explicitly the first three Hermite polynomials:
\begin{equation}\label{eq:firstthree}
h_0(z) = 1, \qquad h_1(z)=z, \qquad h_2(z)=\frac{z^2-1}{\sqrt{2}}.
\end{equation}
The following result, which will be used in the proof of Theorem \ref{th:lowergen}, clarifies in what sense the Hermite expansion is related to analysis on Gaussian spaces. Its proof is a direct consequence of Proposition 11.31 of \cite{o2014analysis}.

\begin{lemma}\label{lemma:hermite}
Let $\sigma, \gamma$ be two functions from $\mathbb R$ to $\mathbb R$ such that $\sigma, \gamma\in L^2(\mathbb R, e^{-x^2/2})$. Then, for any $\bu, \bv\in \mathbb R^d$ s.t. $\norm{\bu}=\norm{\bv}=1$, we have that
\begin{equation}\label{eq:lemmahermite}
\mathbb E\big\{\sigma(\left \langle\bu, \bx\right \rangle)\gamma(\left \langle\bv, \bx\right \rangle) \big\} = \sum_{k\in \mathbb N}\hat{\sigma}_k\hat{\gamma}_k\left \langle \bu, \bv\right \rangle^k,
\end{equation}
where the expectation is with respect to $\bx\sim \normal(\b0,\bI_d)$.
\end{lemma}

Intuitively, the idea of the proof of Theorem \ref{th:lowergen} is to write the generalization error as a sum similar to the RHS of \eqref{eq:lemmahermite}, where $\bu$ is one of the ground-truth weights $\{\bw_i\}_{1\le i \le r}$ and $\bv$ is one of the estimated weights $\{\hat{\bw}_i\}_{1\le i \le R}$. Then, the terms with $k=1$ and $k=2$ do not give any contribution to the generalization error, since the we know that the weights have zero mean and scaled identity covariance. As for the terms with $k\ge 3$, the following lemma gives an upper bound based on the assumption that $|\left \langle \bw_i,\hat{\bw}_j\right \rangle |\le \epsilon$, for all $i\in [r]$ and $j\in [R]$.

\begin{lemma}\label{lemma:corrweights}
Consider weights $\{\bw_i\}_{1\le i\le r}$ that satisfy the assumption {\sf (A4)} for some $\eta_{\rm var}>0$ and weights $\{\hat{\bw}_i\}_{1\le i\le R}$ that satisfy the assumption {\sf (A1)} and are such that, for some $\epsilon\in (0, 1)$, $|\left \langle \bw_i, \hat{\bw}_j\right \rangle |\le \epsilon$, for all $i\in [r]$ and $j\in [R]$. Then, for any integer $k\ge 3$,
\begin{equation}\label{eq:rescorr}
\sum_{i=1}^r\sum_{j=1}^R\left \langle\bw_i, \hat{\bw}_j\right \rangle^k \le \epsilon^{k-2} \cdot  (1+\eta_{\rm var})\cdot\frac{r\cdot R}{d}.
\end{equation}
\end{lemma}

\begin{proof}
The following chain of inequalities holds:
\begin{equation}\label{eq:partcalc}
\begin{split}
\sum_{i=1}^r\sum_{j=1}^R \left \langle\bw_i, \hat{\bw}_j\right \rangle^k &\le\sum_{i=1}^r\sum_{j=1}^R |\left \langle\bw_i, \hat{\bw}_j\right \rangle^k |\\
&\stackrel{\mathclap{\mbox{\footnotesize (a)}}}{\le}\epsilon^{k-2} \sum_{i=1}^r\sum_{j=1}^R \left \langle\bw_i, \hat{\bw}_j\right \rangle^2 \\ 
&=\epsilon^{k-2} \sum_{i=1}^r\sum_{j=1}^R \left \langle \hat{\bw}_j, \bw_i \bw_i^{\sT}\hat{\bw}_j\right \rangle \\ 
&=\epsilon^{k-2} \sum_{j=1}^R \left \langle \hat{\bw}_j, \sum_{i=1}^r\bw_i \bw_i^{\sT}\hat{\bw}_j\right \rangle,
\end{split}
\end{equation}
where in (a) we use that $|\left \langle \bw_i, \hat{\bw}_j\right \rangle |\le \epsilon$ and that $k\ge 3$. Furthermore, as the weights $\{\bw_i\}_{1\le i\le r}$ satisfy the assumption {\sf (A4)}, we immediately have that
\begin{equation}\label{eq:usePDS}
\left \langle \bv, \sum_{i=1}^r\bw_i \bw_i^{\sT}\bv\right \rangle \le (1+\eta_{\rm var})\cdot \frac{r}{d}\cdot \norm{\bv}^2, \qquad \forall\, \bv\in \mathbb R^d.
\end{equation}
As a result, we conclude that
\begin{equation}
\begin{split}
\sum_{i=1}^r\sum_{j=1}^R \left \langle\bw_i, \hat{\bw}_j\right \rangle^k &
\stackrel{\mathclap{\mbox{\footnotesize (a)}}}{\le}\epsilon^{k-2}\cdot (1+\eta_{\rm var})\cdot\frac{r}{d} \sum_{j=1}^R \norm{\hat{\bw}_j}^2 \\
&\stackrel{\mathclap{\mbox{\footnotesize (b)}}}{=} \epsilon^{k-2}\cdot  (1+\eta_{\rm var})\cdot\frac{r\cdot R}{d},
\end{split}
\end{equation}
where in (a) we use \eqref{eq:partcalc} and \eqref{eq:usePDS}, and in (b) we use that the weights $\{\hat{\bw}_i\}_{1\le i\le R}$ satisfy the assumption {\sf (A1)}.
\end{proof}

At this point, we are ready to prove our main result on the generalization error in the setting with separated isotropic weights.

\begin{proof}[Proof of Theorem \ref{th:lowergen}] We divide the proof into three steps. The \emph{first step} consists in showing that
\begin{equation}\label{eq:first}
{\mathbb E}\big\{|y(\bx)-\hat{y}(\bx)|^2\big\} =\sum_{k\in \mathbb N} \hat{\sigma}_k^2\norm{\sum_{i=1}^r\bw_i^{\otimes k}-\sum_{i=1}^R\hat{\bw}_i^{\otimes k}}_F^2.
\end{equation}
This result requires that $\sigma \in L^2(\mathbb R, e^{-x^2/2})$ and that the weights $\{\bw_i\}_{1\le i\le r}$ and $\{\hat{\bw}_i\}_{1\le i \le R}$ satisfy the assumption {\sf (A1)}. Note that \eqref{eq:first} is similar to the claim of Theorem 2.1 of \cite{ge2017learning}.

The \emph{second step} consists in showing that
\begin{equation}\label{eq:second}
\sum_{k\in \mathbb N} \hat{\sigma}_k^2\norm{\sum_{i=1}^r\bw_i^{\otimes k}-\sum_{i=1}^R\hat{\bw}_i^{\otimes k}}_F^2 \ge \sum_{k\ge 3} \hat{\sigma}_k^2\norm{\sum_{i=1}^r\bw_i^{\otimes k}}_F^2\left(1-\frac{2\epsilon^{k-2}\cdot(1+\eta_{\rm var})\cdot\displaystyle\frac{R}{d}}{1-\delta^{k-2}\cdot(1+\eta_{\rm var})\cdot\displaystyle\frac{r}{d}}\right).
\end{equation} 
This result requires that the weights $\{\bw_i\}_{1\le i \le r}$ satisfy the assumptions {\sf (A1)}, {\sf (A2)} and {\sf (A4)} and that the weights $\{\hat{\bw}_i\}_{1\le i \le R}$ satisfy the assumption {\sf (A1)} and have at most $\epsilon$-correlation with $\{\bw_i\}_{1\le i \le r}$.

The \emph{third step} consists in showing that
\begin{equation}\label{eq:third}
\sum_{k\ge 3} \hat{\sigma}_k^2\norm{\sum_{i=1}^r\bw_i^{\otimes k}}_F^2 \ge \min_{a, b\in \mathbb R}{\mathbb E}\left\{\left|y-\left(a+b\norm{\bx}^2\right)\right|^2\right\}-2\hat{\sigma}_1^2 \cdot \eta_{\rm avg}\cdot r - 2\hat{\sigma}_2^2\cdot \eta_{\rm var}^2\cdot \frac{r^2}{d}.
\end{equation}
This result requires that the weights $\{\bw_i\}_{1\le i \le r}$ satisfy the assumptions {\sf (A3)} and {\sf (A4)}.

Note that, for any $k\ge 3$,
\begin{equation}\label{eq:kge3}
1-\frac{2\epsilon^{k-2}\cdot(1+\eta_{\rm var})\cdot\displaystyle\frac{R}{d}}{1-\delta^{k-2}\cdot(1+\eta_{\rm var})\cdot\displaystyle\frac{r}{d}} \ge 1-\frac{2\epsilon\cdot(1+\eta_{\rm var})\cdot\displaystyle\frac{R}{d}}{1-\delta\cdot(1+\eta_{\rm var})\cdot\displaystyle\frac{r}{d}},
\end{equation}
since we can assume that $\epsilon, \delta\in (0, 1)$ without loss of generality. Hence, by putting \eqref{eq:first}, \eqref{eq:second}, \eqref{eq:kge3} and \eqref{eq:third} together, we obtain the lower bound on the generalization error for a generic activation function $\sigma$.

If $\sigma$ is even, then $\hat{\sigma}_k=0$ for $k$ odd. In particular, $\hat{\sigma}_3=0$ and the sum in the RHS of \eqref{eq:second} runs for $k\ge 4$. Note that, for any $k\ge 4$,
\begin{equation}\label{eq:kge4}
1-\frac{2\epsilon^{k-2}\cdot(1+\eta_{\rm var})\cdot\displaystyle\frac{R}{d}}{1-\delta^{k-2}\cdot(1+\eta_{\rm var})\cdot\displaystyle\frac{r}{d}} \ge 1-\frac{2\epsilon^2\cdot(1+\eta_{\rm var})\cdot\displaystyle\frac{R}{d}}{1-\delta^2\cdot(1+\eta_{\rm var})\cdot\displaystyle\frac{r}{d}}.
\end{equation}
Hence, by putting \eqref{eq:first}, \eqref{eq:second}, \eqref{eq:kge4} and \eqref{eq:third} together and by using that $\hat{\sigma}_1=0$, we obtain the lower bound on the generalization error for an even activation function $\sigma$.

\noindent {\bf First step.} By using the definitions \eqref{eq:defy} and \eqref{eq:defhaty}, we have that 
\begin{equation}\label{eq:expansion}
\begin{split}
{\mathbb E}\big\{|y(\bx)-\hat{y}(\bx)|^2\big\}
&= {\mathbb E}\left\{\left|\sum_{i=1}^r \sigma(\left \langle \bx, \bw_i\right \rangle)-\sum_{i=1}^R \sigma(\left \langle \bx, \hat{\bw}_i\right \rangle)\right|^2\right\}\\
&= \sum_{i=1}^r\sum_{j=1}^r{\mathbb E}\left\{ \sigma(\left \langle \bx, \bw_i\right \rangle)\sigma(\left \langle \bx, \bw_j\right \rangle)\right\}-2\sum_{i=1}^r\sum_{j=1}^R{\mathbb E}\left\{ \sigma(\left \langle \bx, \bw_i\right \rangle)\sigma(\left \langle \bx, \hat{\bw}_j\right \rangle)\right\}\\
&\hspace{16em}+\sum_{i=1}^R\sum_{j=1}^R{\mathbb E}\left\{ \sigma(\left \langle \bx, \hat{\bw}_i\right \rangle)\sigma(\left \langle \bx, \hat{\bw}_j\right \rangle)\right\}.
\end{split}
\end{equation}
As $\sigma \in L^2(\mathbb R, e^{-x^2/2})$ and $\norm{{\bw}_i}=\norm{\hat{\bw}_j}=1$ for any $i\in [r]$ and $j\in [R]$, we can apply Lemma \ref{lemma:hermite} and obtain that 
\begin{equation}\label{eq:herm1}
{\mathbb E}\big\{|y(\bx)-\hat{y}(\bx)|^2\big\} =\sum_{k\in \mathbb N} \hat{\sigma}_k^2\left(\sum_{i=1}^r\sum_{j=1}^r\left \langle\bw_i, \bw_j\right \rangle^k - 2\sum_{i=1}^r\sum_{j=1}^R\left \langle\bw_i, \hat{\bw}_j\right \rangle^k+\sum_{i=1}^R\sum_{j=1}^R\left \langle\hat{\bw}_i, \hat{\bw}_j\right \rangle^k\right).
\end{equation}
Note that, for any $k\in \mathbb N$ and any $\bu, \bv\in \mathbb R^d$, 
\begin{equation}\label{eq:tensscal}
\left \langle \bu, \bv \right \rangle^k = \left \langle \bu^{\otimes k}, \bv^{\otimes k}\right \rangle.
\end{equation}
Hence, we can rewrite the RHS of \eqref{eq:herm1} to obtain \eqref{eq:first}. 

\noindent {\bf Second step.} As each term of the sum in the RHS of \eqref{eq:first} is non-negative, we have that
\begin{equation}\label{eq:secfisteq}
\begin{split}
\sum_{k\in \mathbb N} \hat{\sigma}_k^2\norm{\sum_{i=1}^r\bw_i^{\otimes k}-\sum_{i=1}^R\hat{\bw}_i^{\otimes k}}_F^2 
&\ge\sum_{k\ge 3} \hat{\sigma}_k^2\norm{\sum_{i=1}^r\bw_i^{\otimes k}-\sum_{i=1}^R\hat{\bw}_i^{\otimes k}}_F^2\\
&=\sum_{k\ge 3} \hat{\sigma}_k^2 \Bigg(\norm{\sum_{i=1}^r\bw_i^{\otimes k}}_F^2-2\sum_{i=1}^r\sum_{j=1}^R \left \langle \bw_i, \hat{\bw}_j\right \rangle^k+\norm{\sum_{i=1}^R\hat{\bw}_i^{\otimes k}}_F^2\Bigg)\\
&\ge\sum_{k\ge 3} \hat{\sigma}_k^2 \left(\norm{\sum_{i=1}^r\bw_i^{\otimes k}}_F^2-2\sum_{i=1}^r\sum_{j=1}^R \left \langle \bw_i, \hat{\bw}_j\right \rangle^k\right).
\end{split}
\end{equation}
Furthermore, for any $k\ge 3$, the following chain of inequalities holds:
\begin{equation*}
\begin{split}
\norm{\sum_{i=1}^r\bw_i^{\otimes k}}_F^2 &= \sum_{i, j=1}^r \left \langle \bw_i, \bw_j\right \rangle^k \\
&\stackrel{\mathclap{\mbox{\footnotesize (a)}}}{=} r+\sum_{i\neq j} \left \langle \bw_i, \bw_j\right \rangle^k \\
&\ge r-\sum_{i\neq j} |\left \langle \bw_i, \bw_j\right \rangle|^k \\
&\stackrel{\mathclap{\mbox{\footnotesize (b)}}}{\ge} r-\delta^{k-2}\sum_{i\neq j} \left \langle \bw_i, \bw_j\right \rangle^2 \\
&\ge r-\delta^{k-2}\sum_{i, j=1}^r \left \langle \bw_i, \bw_j\right \rangle^2 \\
&= r-\delta^{k-2}\sum_{i=1}^r \left \langle \bw_i, \sum_{j=1}^r\bw_j\bw_j^{\sT}\bw_i\right \rangle, \\
\end{split}
\end{equation*}
where in (a) we use that the weights $\{\bw_i\}_{1\le i\le r}$ satisfy the assumption {\sf (A1)}, and in (b) we use that the weights $\{\bw_i\}_{1\le i\le r}$ satisfy the assumption {\sf (A2)} and that $k\ge 3$. As the weights $\{\bw_i\}_{1\le i\le r}$ satisfy the assumption {\sf (A4)}, \eqref{eq:usePDS} holds. Consequently, for any $k\ge 3$,
\begin{equation}\label{eq:bound1}
\norm{\sum_{i=1}^r\bw_i^{\otimes k}}_F^2 \ge r-\delta^{k-2}\cdot(1+\eta_{\rm var})\cdot\frac{r}{d}\sum_{i=1}^r \norm{\bw_i}^2 =r-\delta^{k-2}\cdot(1+\eta_{\rm var})\cdot\frac{r^2}{d},
\end{equation}
where in the last equality we use again that the weights $\{\bw_i\}_{1\le i\le r}$ satisfy the assumption {\sf (A1)}. By using the hypothesis that $1-\delta\cdot(1+\eta_{\rm var})\cdot r/d\ge 0$, we can rearrange \eqref{eq:bound1} as
\begin{equation}\label{eq:bound1rearr}
r\le \frac{1}{1-\delta^{k-2}\cdot(1+\eta_{\rm var})\cdot\displaystyle\frac{r}{d}}\cdot \norm{\sum_{i=1}^r\bw_i^{\otimes k}}_F^2.
\end{equation}
By combining the result of Lemma \ref{lemma:corrweights} with \eqref{eq:bound1rearr}, we obtain that, for any $k\ge 3$,
\begin{equation*}
\sum_{i=1}^r\sum_{j=1}^R \left \langle \bw_i, \hat{\bw}_j\right \rangle^k \le \frac{\epsilon^{k-2}\cdot(1+\eta_{\rm var})\cdot\displaystyle\frac{R}{d}}{1-\delta^{k-2}\cdot(1+\eta_{\rm var})\cdot\displaystyle\frac{r}{d}} \cdot \norm{\sum_{i=1}^r\bw_i^{\otimes k}}_F^2.
\end{equation*}
Hence, \eqref{eq:second} immediately follows.

\noindent {\bf Third step.} By using the Hermite expansion \eqref{eq:exp} of $\sigma$ and the explicit expression \eqref{eq:firstthree} of the first three Hermite polynomials, we have that
\begin{equation}\label{eq:rewritey}
\begin{split}
y(\bx)&= \hat{\sigma}_0\, r + \hat{\sigma}_1\sum_{i=1}^r\left \langle \bx, \bw_i\right \rangle +\frac{\hat{\sigma}_2}{\sqrt{2}}\sum_{i=1}^r\left(\left \langle \bx, \bw_i\right \rangle^2-1\right)+\sum_{i=1}^r\sum_{k\ge 3} \hat{\sigma}_k h_k(\left \langle \bx, \bw_i\right \rangle)\\
&= \left(\hat{\sigma}_0-\frac{\hat{\sigma}_2}{\sqrt{2}}\right)\cdot r + \hat{\sigma}_1\left \langle \bx, \sum_{i=1}^r\bw_i\right \rangle +\frac{\hat{\sigma}_2}{\sqrt{2}}\left \langle \bx, \sum_{i=1}^r\bw_i\bw_i^{\sT}\bx\right \rangle+\sum_{i=1}^r\sum_{k\ge 3} \hat{\sigma}_k h_k(\left \langle \bx, \bw_i\right \rangle).
\end{split}
\end{equation}
Define
\begin{equation*}
\tilde{y}(\bx)= \left(\hat{\sigma}_0-\frac{\hat{\sigma}_2}{\sqrt{2}}\right)\cdot r +\frac{\hat{\sigma}_2}{\sqrt{2}}\cdot \frac{r}{d}\norm{\bx}^2+\sum_{i=1}^r\sum_{k\ge 3} \hat{\sigma}_k h_k(\left \langle \bx, \bw_i\right \rangle).
\end{equation*}
Then, we immediately have that
\begin{equation}\label{eq:rewritey2}
\begin{split}
\left| y(\bx)-\tilde{y}(\bx)\right|^2
&\le \Bigg|\hat{\sigma}_1\left \langle \bx, \sum_{i=1}^r\bw_i\right \rangle +\frac{\hat{\sigma}_2}{\sqrt{2}}\left \langle \bx, \left(\sum_{i=1}^r\bw_i\bw_i^{\sT}-\frac{r}{d}\bI_d\right)\bx\right \rangle \Bigg|^2\\
&\le 2\Bigg(\hat{\sigma}_1^2\left|\left \langle \bx, \sum_{i=1}^r\bw_i\right \rangle\right|^2 +\frac{\hat{\sigma}_2^2}{2}\left|\left \langle \bx, \left(\sum_{i=1}^r\bw_i\bw_i^{\sT}-\frac{r}{d}\bI_d\right)\bx\right \rangle \right|^2\Bigg).
\end{split}
\end{equation}
Let $\bx\sim \normal(\b0, \bI_d)$. Consequently, for any $\bv\in \mathbb R^d$,
\begin{equation}\label{eq:expvec}
{\mathbb E}\left\{\left|\left \langle \bx, \bv\right \rangle\right|^2\right\} = {\mathbb E}\left\{ \left \langle \bv, \bx\bx^{\sT}\bv\right \rangle \right\}=  \left \langle \bv, {\mathbb E}\left\{\bx\bx^{\sT}\right\}\bv\right \rangle = \norm{\bv}^2,
\end{equation}
where in the last equality we use that ${\mathbb E}\left\{\bx\bx^{\sT}\right\} = \bI_d$. Furthermore, for any $\bA\in \mathbb R^{d\times d}$,
\begin{equation}\label{eq:expmat}
\begin{split}
{\mathbb E}\left\{\left|\left \langle \bx, \bA\bx\right \rangle\right|^2\right\}
&= \sum_{i, j, k, \ell=1}^d A(i, j) A(k, l)\cdot {\mathbb E}\left\{ x(i) x(j) x(k) x(\ell) \right\}\\
&=  \left|{\rm Tr}(\bA)\right|^2 + 2\cdot \norm{\bA}_F^2\\
&\le  \left|{\rm Tr}(\bA)\right|^2 + 2\cdot d\cdot \norm{\bA}_{\rm op}^2,
\end{split}
\end{equation}
where in the last inequality we use that, for any $\bA\in \mathbb R^{d\times d}$,
\begin{equation}
\norm{\bA}_F \le \sqrt{{\rm rank}(A)}\cdot  \norm{\bA}_{\rm op}.
\end{equation}
Note that
\begin{equation}\label{eq:comptrace}
{\rm Tr}\left(\sum_{i=1}^r\bw_i\bw_i^{\sT}-\frac{r}{d}\bI_d\right) = \sum_{i=1}^r\norm{\bw_i}^2 - r= 0.
\end{equation}
Hence, we obtain that 
\begin{equation}\label{eq:rewritey3}
\begin{split}
\left| y(\bx)-\tilde{y}(\bx)\right|^2 
&\stackrel{\mathclap{\mbox{\footnotesize (a)}}}{\le} 2\hat{\sigma}_1^2\norm{\sum_{i=1}^r \bw_i}^2 + 2\hat{\sigma}_2^2\cdot d\norm{\sum_{i=1}^r\bw_i\bw_i^{\sT}-\frac{r}{d}\bI_d}_{\rm op}^2\\
&\stackrel{\mathclap{\mbox{\footnotesize (b)}}}{\le} 2\hat{\sigma}_1^2 \cdot \eta_{\rm avg}\cdot r + 2\hat{\sigma}_2^2\cdot \eta_{\rm var}^2\cdot\frac{r^2}{d},
\end{split}
\end{equation}
where in (a) we combine \eqref{eq:rewritey2}, \eqref{eq:expvec}, \eqref{eq:expmat} and \eqref{eq:comptrace}, and in (b) we use that the weights $\{\bw_i\}_{1\le i\le r}$ satisfy the assumptions {\sf (A3)} and {\sf (A4)}. Furthermore, we have that
\begin{equation}\label{eq:orth}
\begin{split}
{\mathbb E} & \Bigg\{\big( y(\bx)-\tilde{y}(\bx)\big) \left(\tilde{y}(\bx)-\left(\left(\hat{\sigma}_0-\frac{\hat{\sigma}_2}{\sqrt{2}}\right)\cdot r +\frac{\hat{\sigma}_2}{\sqrt{2}}\cdot \frac{r}{d}\norm{\bx}^2\right)\right)\Bigg\} \\
&= {\mathbb E}\Bigg\{\Bigg( \sum_{i=1}^r\sum_{k= 0}^2 \hat{\sigma}_k h_k(\left \langle \bx, \bw_i\right \rangle)-{\mathbb E}\left\{\sum_{i=1}^r\sum_{k= 0}^2 \hat{\sigma}_k h_k(\left \langle \bx, \bg_i\right \rangle)\right\}\Bigg)\sum_{i=1}^r\sum_{k\ge 3} \hat{\sigma}_k h_k(\left \langle \bx, \bw_i\right \rangle)\Bigg\}=0,
\end{split}
\end{equation}
where the inner expectation is with respect to the vectors $\{\bg_i\}_{1\le i \le r}\sim_{\rm i.i.d.}\normal(\b0, \bI_d/d)$. 

Eventually, the following chain of inequalities allows us to conclude:
\begin{equation*}
\begin{split}
\min_{a, b\in \mathbb R}{\mathbb E}&\left\{\left| y(\bx)-\left(a+b\norm{\bx}^2\right)\right |^2\right\} 
\le {\mathbb E}\left\{\left| y(\bx)-\left(\left(\hat{\sigma}_0-\frac{\hat{\sigma}_2}{\sqrt{2}}\right)\cdot r +\frac{\hat{\sigma}_2}{\sqrt{2}}\cdot \frac{r}{d}\norm{\bx}^2\right)\right|^2\right\}\\
&\stackrel{\mathclap{\mbox{\footnotesize (a)}}}{\le} {\mathbb E}\left\{\left|y(\bx)-\tilde{y}(\bx)\right|^2\right\}+{\mathbb E}\left\{\left|\tilde{y}(\bx)-\left(\left(\hat{\sigma}_0-\frac{\hat{\sigma}_2}{\sqrt{2}}\right)\cdot r +\frac{\hat{\sigma}_2}{\sqrt{2}}\cdot \frac{r}{d}\norm{\bx}^2\right)\right|^2\right\}\\
&\stackrel{\mathclap{\mbox{\footnotesize (b)}}}{\le}2\hat{\sigma}_1^2 \cdot \eta_{\rm avg}\cdot r + 2\hat{\sigma}_2^2\cdot \eta_{\rm var}^2\cdot\frac{r^2}{d}+{\mathbb E}\left\{\left|\sum_{i=1}^r\sum_{k\ge 3} \hat{\sigma}_k h_k(\left \langle \bx, \bw_i\right \rangle)\right|^2\right\}\\
&\stackrel{\mathclap{\mbox{\footnotesize (c)}}}{=}2\hat{\sigma}_1^2 \cdot \eta_{\rm avg}\cdot r + 2\hat{\sigma}_2^2\cdot \eta_{\rm var}^2\cdot\frac{r^2}{d}+\sum_{k\ge 3} \hat{\sigma}_k^2\norm{\sum_{i=1}^r\bw_i^{\otimes k}}_F^2,
\end{split}
\end{equation*}
where in (a) we use \eqref{eq:orth}, in (b) we use \eqref{eq:rewritey3}, and (c) is proved by following passages analogous to those of the first step.
\end{proof}

%----------------------------------------------------
\section{Lower Bound on Generalization Error for Random Weights: Proof of Theorem \ref{th:lowergen_G}} \label{app:mainproof_G}
%----------------------------------------------------

As mentioned in Section \ref{sub:iso}, the crucial step is to upper bound the third-order correlation \eqref{eq:3ord}. The idea is to use an epsilon-net argument together with a concentration inequality. One difficulty in deriving the concentration inequality comes from the fact that the weights $\{\bw_i\}_{1\le i\le r}$ are not independent. To circumvent this issue, we first provide an upper bound on 
\begin{equation}\label{eq:3ord_bis}
\frac{1}{R}\sum_{i\le r, j\le R}\<\bg_i,\hat{\bw}_j\>^3.
\end{equation}
This is done in the lemma that immediately follows.

\begin{lemma}\label{lemma:corrweights_Gbis}
Consider weights $\{\bg_i\}_{1\le i \le r}\sim_{\rm i.i.d.}\normal (\b0_d, \bI_d/d)$ and, for $\epsilon\in (0, 1)$, define
\begin{equation}\label{eq:defShatbis}
\hat{\mathcal S}'_\epsilon = \{\{\hat{\bw}_i\}_{1\le i\le R} : \norm{\hat{\bw}_i}=1 \,\,\,\forall\, i\in [R], |\left \langle \bg_i, \hat{\bw}_j\right \rangle |\le \epsilon \,\,\,\forall\,i\in [r]\,\,\,\forall\,j\in [R]\}.
\end{equation}
As $r,d\to\infty$, assume that the conditions \eqref{eq:ass_G} hold. Then, with high probability, for a sequence of vanishing constants $\eta(r, d) = o(1)$,
\begin{equation}\label{eq:rescorr_Gbis}
\sup_{\{\hat{\bw}_i\}_{1\le i\le R} \in \hat{\mathcal S}'_\epsilon}\sum_{i=1}^r\sum_{j=1}^R\left \langle\bg_i, \hat{\bw}_j\right \rangle^3 \le R\cdot \eta(r, d).
\end{equation}
\end{lemma}

\begin{proof}
Note that, as $\epsilon$ decreases, the set $\hat{\mathcal S}'_\epsilon$ contains less elements. Hence, without loss of generality, we can assume that $\epsilon$ is equal to a small constant. For $\bx\in \mathbb R^d$, let
\begin{equation}
h(\bx) = \sum_{i=1}^r\left \langle\bg_i, \bx\right \rangle^3,
\end{equation}
and consider the set ${\sf S}'_\epsilon(\bg_1, \ldots, \bg_r)$ defined as
\begin{equation}\label{eq:defSprimeeps}
{\sf S}'_\epsilon(\bg_1, \ldots, \bg_r) = \{\bx\in {\sf S}^{d-1} : |\left \langle \bx, \bg_i\right \rangle |\le \epsilon \,\,\,\forall i\in [r]\},
\end{equation}
where ${\sf S}^{d-1}$ denotes the set of vectors in $\mathbb R^d$ with unit norm. Then, we have that
\begin{equation}\label{eq:combfin1}
\sup_{\{\hat{\bw}_i\}_{1\le i\le R}\in \hat{\mathcal S}'_\epsilon}\sum_{i=1}^r\sum_{j=1}^R\left \langle\bg_i, \hat{\bw}_j\right \rangle^3 =\sup_{\{\hat{\bw}_i\}_{1\le i\le R}\in \hat{\mathcal S}'_\epsilon} \sum_{j=1}^R h(\hat{\bw}_j)\le R \max_{\bx\in {\sf S}'_\epsilon(\bg_1, \ldots, \bg_r) }h(\bx).
\end{equation}

Let $N^d(\epsilon)$ be an $\epsilon$-net of ${\sf S}^{d-1}$. This means that any point in ${\sf S}^{d-1}$ has distance at most $\epsilon$ from $N^d(\epsilon)$.  As 
${\sf S}'_\epsilon(\bg_1, \ldots, \bg_r) \subseteq {\sf S}^{d-1}$, we also have that any point in ${\sf S}'_\epsilon(\bg_1, \ldots, \bg_r)$ has distance at most $\epsilon$ from $N^d(\epsilon)$. Remove from $N^d(\epsilon)$ all the points that are not at distance at most $\epsilon$ from ${\sf S}'_\epsilon(\bg_1, \ldots, \bg_r)$ and call the remaining set $\tilde{N}^d(\epsilon)$. Hence, we have that
\begin{equation}\label{eq:scalx1}
\max_{\bx'\in \tilde{N}^d(\epsilon)}\max_{i\in [r]}|\left \langle \bx', \bg_i\right \rangle |\le\max_{\bx\in {\sf S}'_\epsilon(\bg_1, \ldots, \bg_r)}\max_{i\in [r]}|\left \langle \bx, \bg_i\right \rangle | + \epsilon\max_{i\in [r]}\norm{\bg_i}\le 3\epsilon,
\end{equation}
where the last inequality holds with high probability. Furthermore, for any $\bx'\in \tilde{N}^d(\epsilon)$,
\begin{equation}\label{eq:scalx2}
1-\epsilon\le\norm{\bx'}\le 1+ \epsilon.
\end{equation}

Note that, for any $\bx \in {\sf S}'_\epsilon(\bg_1, \ldots, \bg_r)$, there exists $\bx'\in \tilde{N}^d(\epsilon)$, $\alpha \in (0, \epsilon]$, and $\tilde{\bx}\in {\sf S}^{d-1}$ such that
\begin{equation}
\bx = \bx'+\alpha \tilde{\bx},
\end{equation}
which immediately implies that
\begin{equation}\label{eq:scalx3}
h(\bx) = \norm{\bx'}^3 h\left(\frac{\bx'}{\norm{\bx'}}\right) + 3\norm{\bx'}^2\alpha\sum_{i=1}^r\left \langle\bg_i, \frac{\bx'}{\norm{\bx'}}\right \rangle^2\left \langle\bg_i, \tilde{\bx}\right \rangle + 3\norm{\bx'}\alpha^2\sum_{i=1}^r\left \langle\bg_i, \frac{\bx'}{\norm{\bx'}}\right \rangle\left \langle\bg_i, \tilde{\bx}\right \rangle^2 + \alpha^3 h(\tilde{\bx}).
\end{equation}
Furthermore, we have that
\begin{equation}\label{eq:scalx4}
\max_{\substack{\bx'\in \tilde{N}^d(\epsilon) \\ \tilde{\bx}\in {\sf S}^{d-1}}}\sum_{i=1}^r\left \langle\bg_i, \frac{\bx'}{\norm{\bx'}}\right \rangle^2\left \langle\bg_i, \tilde{\bx}\right \rangle \le \max_{\bx'', \tilde{\bx}\in {\sf S}^{d-1}}\sum_{i=1}^r\left \langle\bg_i, \bx''\right \rangle^2\left \langle\bg_i, \tilde{\bx}\right \rangle = \max_{\bx''\in {\sf S}^{d-1}}\sum_{i=1}^r\left \langle\bg_i, \bx''\right \rangle^3, 
\end{equation}
where the last equality is a consequence of Theorem 6.9 of \cite{hillar2013most}. Similarly, we have that
\begin{equation}\label{eq:scalx5}
\max_{\substack{\bx'\in \tilde{N}^d(\epsilon) \\ \tilde{\bx}\in {\sf S}^{d-1}}}\sum_{i=1}^r\left \langle\bg_i, \frac{\bx'}{\norm{\bx'}}\right \rangle\left \langle\bg_i, \tilde{\bx}\right \rangle^2 \le \max_{\bx''\in {\sf S}^{d-1}}\sum_{i=1}^r\left \langle\bg_i, \bx''\right \rangle^3. 
\end{equation}
By putting \eqref{eq:scalx2}, \eqref{eq:scalx3}, \eqref{eq:scalx4}, and \eqref{eq:scalx5} together, we conclude that 
\begin{equation}\label{eq:scalxfin}
\max_{\bx\in {\sf S}'_\epsilon(\bg_1, \ldots, \bg_r) }h(\bx) \le 2\max_{\bx\in \tilde{N}^d(\epsilon) }h\left(\frac{\bx}{\norm{\bx}}\right) + c_1 \epsilon \max_{\bx\in {\sf S}^{d-1}}h(\bx), 
\end{equation}
for some constant $c_1$ which does not depend on $\epsilon$. Furthermore, by using \eqref{eq:scalx1} and \eqref{eq:scalx2}, we deduce that, with high probability,
\begin{equation}\label{eq:scalx6}
\max_{\bx'\in \tilde{N}^d(\epsilon)}\max_{i\in [r]}\left|\left \langle \frac{\bx'}{\norm{\bx'}}, \bg_i\right \rangle \right|\le 4\epsilon.
\end{equation}
Define $\epsilon'=4\epsilon$ and, for $\bx\in \mathbb R^d$, let
\begin{equation}
\bar{h}(\bx) = \sum_{i=1}^rh_{\epsilon'}(\left \langle\bg_i, \bx\right \rangle),
\end{equation}
where
\begin{equation}\label{eq:defhesp}
h_{\epsilon'}(x) = \left\{
\begin{array}{ll}
x^3 & \mbox{ if }|x|\le \epsilon',\\
-\left(\epsilon'\right)^3 & \mbox{ if }x\le -\epsilon',\\
\left(\epsilon'\right)^3 & \mbox{ if }x\ge \epsilon'.\\
\end{array}\right.
\end{equation}
Consequently, by using \eqref{eq:scalx6}, it is clear that, for any $\bx\in \tilde{N}^d(\epsilon)$,
\begin{equation}\label{eq:equalityfun}
h\left(\frac{\bx}{\norm{\bx}}\right) = \bar{h}\left(\frac{\bx}{\norm{\bx}}\right).
\end{equation}

Given $\bx\in {\sf S}^{d-1}$, let us provide an upper bound on $\mathbb P(|\bar{h}(\bx)| > t)$. First, note that $h_{\epsilon'}$ is odd and the distribution of $\bg_i$ is symmetric. Then, 
\begin{equation}\label{eq:boundpoint}
\mathbb P(|\bar{h}(\bx)| > t)= 2\,\mathbb P(\bar{h}(\bx) > t).
\end{equation}
Note also that the random variables $\{h_{\epsilon'}(\left \langle\bg_i, \bx\right \rangle)\}_{1\le i \le r}$ are independent and identically distributed. Hence, by Chernoff bound, we have that, for any $\lambda>0$, the RHS of \eqref{eq:boundpoint} is upper bounded by
\begin{equation}\label{eq:boundpointgen}
2e^{-\lambda t}\left(\mathbb E\left\{e^{\lambda h_{\epsilon'}(\left \langle\bg_1, \bx\right \rangle)}\right\}\right)^r.
\end{equation}
Pick $\lambda = d/(4\epsilon')$. Then, \eqref{eq:boundpointgen} is rewritten as 
\begin{equation}\label{eq:boundpoint2}
2\exp\left(-\frac{d t}{4\epsilon'}\right)\left(\mathbb E\left\{\exp\left(\frac{d h_{\epsilon'}(\left \langle\bg_1, \bx\right \rangle)}{4\epsilon'}\right)\right\}\right)^r.
\end{equation}
Since $\bg_1\sim \normal(\b0_d, \bI_d/d)$, we obtain that $\left \langle\bg_1, \bx\right \rangle=G\sim \normal(0, 1/d)$ for any $\bx\in {\sf S}^{d-1}$. Hence, 
\begin{equation}\label{eq:intsimpl1}
\begin{split}
\mathbb E\left\{\exp\left(\frac{d h_{\epsilon'}(\left \langle\bg_1, \bx\right \rangle)}{4\epsilon'}\right)\right\} &= \int_{-\epsilon'\sqrt{d}}^{\epsilon'\sqrt{d}} \frac{1}{\sqrt{2\pi}}\exp\left(\frac{x^3}{4\epsilon'\sqrt{d}}-\frac{x^2}{2}\right)\,{\rm d}x \\
&+ \exp\left(\frac{d (\epsilon')^{2}}{4}\right)\mathbb P(G > \epsilon')+ \exp\left(-\frac{d (\epsilon')^{2}}{4}\right)\mathbb P(G < -\epsilon').
\end{split}
\end{equation}
It is easy to see that 
\begin{equation}
\exp\left(\frac{d (\epsilon')^{2}}{4}\right)\mathbb P(G > \epsilon')+ \exp\left(-\frac{d (\epsilon')^{2}}{4}\right)\mathbb P(G < -\epsilon')
\le 2 \exp\left(-\frac{d(\epsilon')^2}{4}\right).
\end{equation}
Furthermore, we have that
\begin{equation}\label{eq:2intdiv}
\begin{split}
\int_{-\epsilon'\sqrt{d}}^{\epsilon'\sqrt{d}} \frac{1}{\sqrt{2\pi}}\exp\left(\frac{ x^3}{4\epsilon'\sqrt{d}}-\frac{x^2}{2}\right)\,{\rm d}x &\le \int_{-\sqrt{8\log d}}^{\sqrt{8\log d}} \frac{1}{\sqrt{2\pi}}\exp\left(\frac{x^3}{4\epsilon'\sqrt{d}}-\frac{x^2}{2}\right)\,{\rm d}x \\
&+ 2\int_{\sqrt{8\log d}}^{\epsilon'\sqrt{d}} \frac{1}{\sqrt{2\pi}}\exp\left(\frac{ x^3}{4\epsilon'\sqrt{d}}-\frac{x^2}{2}\right)\,{\rm d}x.
\end{split}
\end{equation}
The second integral in the RHS of \eqref{eq:2intdiv} is upper bounded as follows:
\begin{equation}
\begin{split}
\int_{\sqrt{8\log d}}^{\epsilon'\sqrt{d}} \frac{1}{\sqrt{2\pi}}\exp\left(\frac{x^3}{4\epsilon'\sqrt{d}}-\frac{x^2}{2}\right)\,{\rm d}x & \le  \int_{\sqrt{8\log d}}^{\epsilon'\sqrt{d}} \frac{1}{\sqrt{2\pi}}e^{-\frac{x^2}{4}}\,{\rm d}x\\
& \le \int_{\sqrt{8\log d}}^{+\infty} \frac{1}{\sqrt{2\pi}}e^{-\frac{x^2}{4}}\,{\rm d}x\le \frac{\sqrt{2}}{d^2}.
\end{split}
\end{equation}
In order to upper bound the first integral in the RHS of \eqref{eq:2intdiv}, we define
\begin{equation}
\phi(\lambda) = \log\left({\mathbb E}\left\{\exp\left(\frac{\lambda \tilde{G}^3}{\sqrt{d}}\right)\right\}\right),
\end{equation}
where the probability density function of $\tilde{G}$ is given by  
\begin{equation}
p(\tilde{g}) = \left\{\begin{array}{ll}
\frac{\displaystyle \frac{1}{\sqrt{2\pi}}e^{-\frac{\tilde{g}^2}{2}}}{\displaystyle\int_{-\sqrt{8\log d}}^{\sqrt{8\log d}} \frac{1}{\sqrt{2\pi}}e^{-\frac{x^2}{2}}\,{\rm d}x}, & \mbox{ if } |\tilde{g}|\le \sqrt{8\log d},\\
&\\
0, & \mbox{ otherwise.}
\end{array}\right.
\end{equation}
Then, we immediately have that
\begin{equation}
\int_{-\sqrt{8\log d}}^{\sqrt{8\log d}} \frac{1}{\sqrt{2\pi}}\exp\left(\frac{x^3}{4\epsilon'\sqrt{d}}-\frac{x^2}{2}\right)\,{\rm d}x =\exp\left(\phi\left(\frac{1}{4\epsilon'}\right)\right)\int_{-\sqrt{8\log d}}^{\sqrt{8\log d}} \frac{1}{\sqrt{2\pi}}e^{-\frac{x^2}{2}}\,{\rm d}x \le \exp\left(\phi\left(\frac{1}{4\epsilon'}\right)\right).
\end{equation}
After some calculations, we obtain that, for $d$ sufficiently large,
\begin{equation}
\begin{split}
\phi(0) &= 0,\\
\phi'(0) &= 0,\\
\phi''(0) &\le \frac{30}{d},\\
\phi'''(\lambda)&\le c_2 \frac{(\log d)^{9/2}}{d^{3/2}}, 
\end{split}
\end{equation}
for some constant $c_2$ which does not depend on $d$. Consequently, by Taylor's inequality, we deduce that 
\begin{equation}
\phi\left(\frac{1}{4\epsilon'}\right)\le \frac{15}{16d(\epsilon')^2} + \frac{c_2}{384(\epsilon')^3} \frac{(\log d)^{9/2}}{d^{3/2}},
\end{equation}
which implies that
\begin{equation}\label{eq:boundexpfull}
\begin{split}
\mathbb E\left\{\exp\left(\frac{d h_{\epsilon'}(\left \langle\bg_1, \bx\right \rangle)}{4\epsilon'}\right)\right\}&\le \exp\left(\frac{15}{16d(\epsilon')^2} + \frac{c_2}{384(\epsilon')^3} \frac{(\log d)^{9/2}}{d^{3/2}}\right) + \frac{2\sqrt{2}}{d^2} + 2 \exp\left(-\frac{d(\epsilon')^2}{4}\right)\\
&\le \exp\left(\frac{15}{16d(\epsilon')^2} + \frac{c_2}{384(\epsilon')^3} \frac{(\log d)^{9/2}}{d^{3/2}}\right)\left(1+\frac{2\sqrt{2}}{d^2} + 2 \exp\left(-\frac{d(\epsilon')^2}{4}\right)\right).
\end{split}
\end{equation}
By putting \eqref{eq:boundpoint}, \eqref{eq:boundpoint2}, and \eqref{eq:boundexpfull} together, we conclude that
\begin{equation}\label{eq:boundpointnew}
\mathbb P(|\bar{h}(\bx)| > t) \le 2\exp\left(-\frac{d t}{4\epsilon'} +r\left(\frac{15}{16d(\epsilon')^2} + \frac{c_2}{384(\epsilon')^3} \frac{(\log d)^{9/2}}{d^{3/2}}+\frac{2\sqrt{2}}{d^2} + 2 \exp\left(-\frac{d(\epsilon')^2}{4}\right)\right)\right),
\end{equation}
where we have also used that $\log(1+x)\le x$ for any $x\ge 0$.

By using \eqref{eq:equalityfun} and \eqref{eq:boundpointnew} together with a union bound over the points of the set $\tilde{N}^d(\epsilon)$, we obtain that
\begin{equation}\label{eq:ubfinalbd}
\begin{split}
\mathbb P&\left(\max_{\bx\in \tilde{N}^d(\epsilon)}\left|h\left(\frac{\bx}{\norm{\bx}}\right)\right| > t\right)\\
&\le 2\exp\left(-\frac{d t}{4\epsilon'} +r\left(\frac{15}{16d(\epsilon')^2} + \frac{c_2}{384(\epsilon')^3} \frac{(\log d)^{9/2}}{d^{3/2}}+\frac{2\sqrt{2}}{d^2} + 2 \exp\left(-\frac{d(\epsilon')^2}{4}\right)\right)\right) |\tilde{N}^d(\epsilon)|\\
&\le 2\exp\left(-\frac{d t}{4\epsilon'} +r\left(\frac{15}{16d(\epsilon')^2} + \frac{c_2}{384(\epsilon')^3} \frac{(\log d)^{9/2}}{d^{3/2}}+\frac{2\sqrt{2}}{d^2} + 2 \exp\left(-\frac{d(\epsilon')^2}{4}\right)\right)\right) \left(1+\frac{8}{\epsilon'}\right)^d,
\end{split}
\end{equation}
where in the last inequality we use that $|\tilde{N}^d(\epsilon)|\le |N^d(\epsilon)|$ and that there exists an $\epsilon$-net of ${\sf S}^{d-1}$ that contains at most $(1+2/\epsilon)^d$ points, see Lemma 5.2 of \cite{vershynin2010introduction}. By using that $\epsilon'=4\epsilon=o(1)$ and that $r =o(d^2)$, we obtain that, with high probability,
\begin{equation}\label{eq:combfin2}
\max_{\bx\in \tilde{N}^d(\epsilon) }h\left(\frac{\bx}{\norm{\bx}}\right) = o(1).
\end{equation}

We now prove that $\max_{\bx\in {\sf S}^{d-1}}h(\bx)$ is upper bounded by a constant by using another epsilon-net argument. Let $N_1^d(\delta)$ be a $\delta$-net of ${\sf S}^{d-1}$. Remove from $N_1^d(\delta)$ all the points that are not at distance at most $\delta$ from ${\sf S}^{d-1}$ and call the remaining set $\tilde{N}_1^d(\delta)$. 

By following the same argument that yields \eqref{eq:scalxfin}, we obtain that 
\begin{equation}\label{eq:scalxfin2}
\max_{\bx\in {\sf S}^{d-1} }h(\bx) \le 2\max_{\bx\in \tilde{N}_1^d(\delta) }h\left(\frac{\bx}{\norm{\bx}}\right) + c_3 \delta \max_{\bx\in {\sf S}^{d-1}}h(\bx), 
\end{equation}
for some constant $c_3$. Set $\delta=1/(2c_3)$. Then, we can rearrange \eqref{eq:scalxfin2} as
\begin{equation}
\max_{\bx\in {\sf S}^{d-1} }h(\bx) \le 4\max_{\bx\in N_1^d(\delta) }h\left(\frac{\bx}{\norm{\bx}}\right). 
\end{equation}
Note that, with high probability,
\begin{equation}\label{eq:scalx6bis}
\max_{\bx\in \tilde{N}_1^d(\delta)}\max_{i\in [r]}\left|\left \langle \frac{\bx}{\norm{\bx}}, \bg_i\right \rangle \right|\le 2.
\end{equation}
Hence, by following the same argument that yields \eqref{eq:ubfinalbd} with $\epsilon'=2$, we obtain that 
\begin{equation}
\mathbb P\left(\max_{\bx\in \tilde{N}_1^d(\delta)}\left|h\left(\frac{\bx}{\norm{\bx}}\right)\right| > t\right)\le 2\exp\left(-\frac{d t}{8} +r\left(\frac{15}{64d} + \frac{c_2}{3072} \frac{(\log d)^{9/2}}{d^{3/2}}+\frac{2\sqrt{2}}{d^2} + 2 e^{-d}\right)\right) \left(1+4c_3\right)^d.
\end{equation}
By using that $r =o(d^2)$, we deduce that, with high probability,
\begin{equation}\label{eq:combfin3}
\max_{\bx\in {\sf S}^{d-1}}h(\bx)\le 16\log(1+4c_3).
\end{equation}
By combining \eqref{eq:combfin1}, \eqref{eq:scalxfin}, \eqref{eq:combfin2}, \eqref{eq:combfin3} with the fact that $\epsilon=o(1)$, the result follows.
\end{proof}

Next, we provide an upper bound on all the higher-order correlations
\begin{equation}\label{eq:3ord_all}
\sup_{k\ge 3}\frac{1}{R}\sum_{i\le r, j\le R}\<\bg_i,\hat{\bw}_j\>^k.
\end{equation}

\begin{lemma}\label{lemma:corrweights_Gall}
Consider weights $\{\bg_i\}_{1\le i \le r}\sim_{\rm i.i.d.}\normal (\b0_d, \bI_d/d)$ and, for $\epsilon\in (0, 1)$, define $\hat{\mathcal S}'_\epsilon$ as in \eqref{eq:defShatbis}. As $r,d\to\infty$, assume that the conditions \eqref{eq:ass_G} hold. Then, with high probability, for a sequence of vanishing constants $\eta(r, d) = o(1)$,
\begin{equation}\label{eq:rescorr_Gall}
\sup_{k\ge 3}\sup_{\{\hat{\bw}_i\}_{1\le i\le R} \in \hat{\mathcal S}'_\epsilon}\sum_{i=1}^r\sum_{j=1}^R\left \langle\bg_i, \hat{\bw}_j\right \rangle^k \le R\cdot \eta(r, d).
\end{equation}
\end{lemma}

\begin{proof}
By definition of $\hat{\mathcal S}'_\epsilon$, we immediately have that
\begin{equation}\label{eq:supk4}
\sup_{k\ge 4}\sup_{\{\hat{\bw}_i\}_{1\le i\le R} \in \hat{\mathcal S}'_\epsilon}\sum_{i=1}^r\sum_{j=1}^R\left \langle\bg_i, \hat{\bw}_j\right \rangle^k \le \sup_{\{\hat{\bw}_i\}_{1\le i\le R} \in \hat{\mathcal S}'_\epsilon}\sum_{i=1}^r\sum_{j=1}^R\left \langle\bg_i, \hat{\bw}_j\right \rangle^4. 
\end{equation}
In order to bound the RHS of \eqref{eq:supk4}, we follow an argument similar to that of the proof of Lemma \ref{lemma:corrweights_Gbis}. 

For $\bx\in \mathbb R^d$, let
\begin{equation}
q(\bx) = \sum_{i=1}^r\left \langle\bg_i, \bx\right \rangle^4,
\end{equation}
and consider the set ${\sf S}'_\epsilon(\bg_1, \ldots, \bg_r)$ defined as in \eqref{eq:defSprimeeps}. Then, we have that 
\begin{equation}\label{eq:combfin1_k4}
\sup_{\{\hat{\bw}_i\}_{1\le i\le R}\in \hat{\mathcal S}'_\epsilon}\sum_{i=1}^r\sum_{j=1}^R\left \langle\bg_i, \hat{\bw}_j\right \rangle^4 \le R \max_{\bx\in {\sf S}'_\epsilon(\bg_1, \ldots, \bg_r) }q(\bx).
\end{equation}
Let $N^d(\epsilon)$ be an $\epsilon$-net of ${\sf S}^{d-1}$. Remove from $N^d(\epsilon)$ all the points that are not at distance at most $\epsilon$ from ${\sf S}'_\epsilon(\bg_1, \ldots, \bg_r)$ and call the remaining set $\tilde{N}^d(\epsilon)$. By following the same argument that yields \eqref{eq:scalxfin}, we obtain that
\begin{equation}\label{eq:scalxfin_k4}
\max_{\bx\in {\sf S}'_\epsilon(\bg_1, \ldots, \bg_r) }q(\bx) \le 2\max_{\bx\in \tilde{N}^d(\epsilon) }q\left(\frac{\bx}{\norm{\bx}}\right) + c_1 \epsilon \max_{\bx\in {\sf S}^{d-1}}q(\bx), 
\end{equation}
for some constant $c_1$ which does not depend on $\epsilon$. Furthermore, with high probability, \eqref{eq:scalx6} holds. Define $\epsilon'=4\epsilon$ and, for $\bx\in \mathbb R^d$, let
\begin{equation}
\bar{q}(\bx) = \sum_{i=1}^rq_{\epsilon'}(\left \langle\bg_i, \bx\right \rangle),
\end{equation}
where
\begin{equation}\label{eq:defhesp_k4}
q_{\epsilon'}(x) = \left\{
\begin{array}{ll}
x^4 & \mbox{ if }|x|\le \epsilon',\\
\left(\epsilon'\right)^4 & \mbox{ if }|x|\ge \epsilon'.\\
\end{array}\right.
\end{equation}
Consequently, by using \eqref{eq:scalx6}, it is clear that, for any $\bx\in \tilde{N}^d(\epsilon)$,
\begin{equation}\label{eq:equalityfun_k4}
q\left(\frac{\bx}{\norm{\bx}}\right) = \bar{q}\left(\frac{\bx}{\norm{\bx}}\right).
\end{equation}

Given $\bx\in {\sf S}^{d-1}$, let us provide an upper bound on $\mathbb P(\bar{q}(\bx) > t)$. By Chernoff bound, we have that
\begin{equation}\label{eq:boundpoint2_k4}
\mathbb P(\bar{q}(\bx) > t)\le \exp\left(-\frac{d t}{4(\epsilon')^2}\right)\left(\mathbb E\left\{\exp\left(\frac{d q_{\epsilon'}(\left \langle\bg_1, \bx\right \rangle)}{4(\epsilon')^2}\right)\right\}\right)^r.
\end{equation}
As $\left \langle\bg_1, \bx\right \rangle=G\sim \normal(0, 1/d)$, we obtain that
\begin{equation}\label{eq:intsimpl1_k4}
\mathbb E\left\{\exp\left(\frac{d q_{\epsilon'}(\left \langle\bg_1, \bx\right \rangle)}{4(\epsilon')^2}\right)\right\} = 2\int_{0}^{\epsilon'\sqrt{d}} \frac{1}{\sqrt{2\pi}}\exp\left(\frac{x^4}{4(\epsilon')^2d}-\frac{x^2}{2}\right)\,{\rm d}x +2 \exp\left(\frac{d (\epsilon')^{2}}{4}\right)\mathbb P(G > \epsilon').
\end{equation}
It is easy to see that 
\begin{equation}\label{eq:intsimpl2_k4}
2\exp\left(\frac{d (\epsilon')^{2}}{4}\right)\mathbb P(G > \epsilon')
\le 2 \exp\left(-\frac{d(\epsilon')^2}{4}\right).
\end{equation}
Furthermore, we have that
\begin{equation}\label{eq:2intdiv_k4}
\begin{split}
2\int_{0}^{\epsilon'\sqrt{d}} \frac{1}{\sqrt{2\pi}}\exp\left(\frac{ x^4}{4(\epsilon')^2d}-\frac{x^2}{2}\right)\,{\rm d}x &\le 2\int_{0}^{\sqrt{8\log d}} \frac{1}{\sqrt{2\pi}}\exp\left(\frac{x^4}{4(\epsilon')^2d}-\frac{x^2}{2}\right)\,{\rm d}x \\
&+ 2\int_{\sqrt{8\log d}}^{\epsilon'\sqrt{d}} \frac{1}{\sqrt{2\pi}}\exp\left(\frac{ x^4}{4(\epsilon')^2d}-\frac{x^2}{2}\right)\,{\rm d}x.
\end{split}
\end{equation}
The second integral in the RHS of \eqref{eq:2intdiv_k4} is upper bounded as
\begin{equation}\label{eq:cher2_k4}
\int_{\sqrt{8\log d}}^{\epsilon'\sqrt{d}} \frac{1}{\sqrt{2\pi}}\exp\left(\frac{x^4}{4(\epsilon')^2d}-\frac{x^2}{2}\right)\,{\rm d}x  \le  \int_{\sqrt{8\log d}}^{\epsilon'\sqrt{d}} \frac{1}{\sqrt{2\pi}}e^{-\frac{x^2}{4}}\,{\rm d}x\le \frac{\sqrt{2}}{d^2}.
\end{equation}
The first integral in the RHS of \eqref{eq:2intdiv_k4} is upper bounded as
\begin{equation}\label{eq:cher3_k4}
\begin{split}
2\int_{0}^{\sqrt{8\log d}} \frac{1}{\sqrt{2\pi}}\exp\left(\frac{x^4}{4(\epsilon')^2d}-\frac{x^2}{2}\right)\,{\rm d}x&\le 2\int_{0}^{\sqrt{8\log d}} \frac{1}{\sqrt{2\pi}}\exp\left(x^2\left(-\frac{1}{2}+\frac{2\log d}{(\epsilon')^2d}\right)\right)\,{\rm d}x\\
&\le 1+\frac{c_2}{(\epsilon')^2}\frac{\log d}{d},
\end{split}
\end{equation}
for some constant $c_2$.
By putting \eqref{eq:boundpoint2_k4}- \eqref{eq:cher3_k4} together, we conclude that
\begin{equation}\label{eq:boundpointnew_k4}
\mathbb P(|\bar{q}(\bx)| > t) \le \exp\left(-\frac{d t}{4(\epsilon')^2} +r\left(\frac{c_2}{(\epsilon')^2}\frac{\log d}{d}+\frac{2\sqrt{2}}{d^2} + 2 \exp\left(-\frac{d(\epsilon')^2}{4}\right)\right)\right),
\end{equation}
where we have also used that $\log(1+x)\le x$ for any $x\ge 0$.

Recall that $\epsilon'=4\epsilon=o(1)$, $r =o(d^2/\log d)$, and that 
$|\tilde{N}^d(\epsilon)|\le (1+2/\epsilon)^d$ (see Lemma 5.2 of  \cite{vershynin2010introduction}). Then, by performing a union bound over the points of the set $\tilde{N}^d(\epsilon)$, we conclude that, with high probability,
\begin{equation}\label{eq:combfin2_k4}
\max_{\bx\in \tilde{N}^d(\epsilon) }q\left(\frac{\bx}{\norm{\bx}}\right) = o(1).
\end{equation}

We now prove that $\max_{\bx\in {\sf S}^{d-1}}q(\bx)$ is upper bounded by a constant by using another epsilon-net argument. Let $N_1^d(\delta)$ be a $\delta$-net of ${\sf S}^{d-1}$. Remove from $N_1^d(\delta)$ all the points that are not at distance at most $\delta$ from ${\sf S}^{d-1}$ and call the remaining set $\tilde{N}_1^d(\delta)$. By following the same argument that yields \eqref{eq:scalxfin_k4}, we obtain that 
\begin{equation}\label{eq:scalxfin2_k4}
\max_{\bx\in {\sf S}^{d-1} }q(\bx) \le 2\max_{\bx\in \tilde{N}_1^d(\delta) }q\left(\frac{\bx}{\norm{\bx}}\right) + c_3 \delta \max_{\bx\in {\sf S}^{d-1}}q(\bx), 
\end{equation}
for some constant $c_3$. Set $\delta=1/(2c_3)$. Then, we can rearrange \eqref{eq:scalxfin2_k4} as
\begin{equation}
\max_{\bx\in {\sf S}^{d-1} }q(\bx) \le 4\max_{\bx\in N_1^d(\delta) }q\left(\frac{\bx}{\norm{\bx}}\right). 
\end{equation}
Note that, with high probability, \eqref{eq:scalx6bis} holds. Hence, by following the same argument that yields \eqref{eq:boundpointnew_k4} with $\epsilon'=2$, we obtain that 
\begin{equation}\label{eq:ubfinalbd_k4}
\mathbb P\left(\max_{\bx\in \tilde{N}_1^d(\delta)}\left|q\left(\frac{\bx}{\norm{\bx}}\right)\right| > t\right)\le 2\exp\left(-\frac{d t}{16} +r\left(\frac{c_2}{4} \frac{\log d}{d}+\frac{2\sqrt{2}}{d^2} + 2 e^{-d}\right)\right) \left(1+4c_3\right)^d.
\end{equation}
By using that $r =o(d^2/\log d)$, we deduce that, with high probability,
\begin{equation}\label{eq:combfin3_k4}
\max_{\bx\in {\sf S}^{d-1}}q(\bx)\le 32\log(1+4c_3).
\end{equation}
By combining \eqref{eq:combfin1_k4}, \eqref{eq:scalxfin_k4}, \eqref{eq:combfin2_k4}, \eqref{eq:combfin3_k4} with the fact that $\epsilon=o(1)$, we conclude that, with high probability, the RHS of \eqref{eq:supk4} is $o(1)$. By using Lemma \ref{lemma:corrweights_Gbis}, the proof is complete.

\end{proof}

At this point, we are ready to provide an upper bound on the correlations
\begin{equation}\label{eq:3ord_all_bis}
\sup_{k\ge 3}\frac{1}{R}\sum_{i\le r, j\le R}\<\bw_i,\hat{\bw}_j\>^k.
\end{equation}
The idea is to show that the quantity in \eqref{eq:3ord_all_bis} is close to the quantity in \eqref{eq:3ord_all}, and then to apply Lemma \ref{lemma:corrweights_Gall}.

\begin{lemma}\label{lemma:corrweights_G}
Consider weights $\{\bw_i\}_{1\le i\le r}$ of the form \eqref{eq:defrandomw} and, for some $\epsilon\in (0, 1)$, define $\hat{\mathcal S}_\epsilon$ as in \eqref{eq:defShat}. As $r,d\to\infty$, assume that the conditions \eqref{eq:ass_G} hold. Then, with high probability, for a sequence of vanishing constants $\eta(r, d) = o(1)$,
\begin{equation}\label{eq:rescorr_G}
\sup_{k\ge 3}\sup_{\{\hat{\bw}_i\}_{1\le i\le R} \in \hat{\mathcal S}_\epsilon}\sum_{i=1}^r\sum_{j=1}^R\left \langle\bw_i, \hat{\bw}_j\right \rangle^k \le R\cdot \eta(r, d).
\end{equation}
\end{lemma}

\begin{proof}
A trivial upper bound on the RHS of \eqref{eq:rescorr_G} is given by $\epsilon^3\,R\,r$. Hence, without loss of generality, we can assume that
\begin{equation}\label{eq:wlogup}
\epsilon \ge \frac{2}{\sqrt{r}}.
\end{equation}
Furthermore, as in the proof of Lemma \ref{lemma:corrweights_Gbis}, we can also assume that $\epsilon$ is equal to a small constant.

Recall the definition \eqref{eq:defrandomw} of the weights $\{\bw_i\}_{1\le i\le r}$ and, for $i\in [r]$, let
\begin{equation}\label{eq:defrandomwextra}
\begin{split}
\bg_{\rm avg}&=\frac{1}{r}\sum_{i=1}^r \bg_i,\\
\tilde{\bg}_i&=\bg_i - \bg_{\rm avg},\\
c_i&=\frac{1}{\norm{\tilde{\bg}_i}}.\\
\end{split}
\end{equation}
Then, with high probability,
\begin{equation}\label{eq:scatildeg2}
\sup_{\{\hat{\bw}_i\}_{1\le i\le R} \in \hat{\mathcal S}_\epsilon}\max_{j\in [r]}|\left \langle \hat{\bw}_i, \bg_j\right \rangle |\le \norm{\bg_{\rm avg}} + \sup_{\{\hat{\bw}_i\}_{1\le i\le R} \in \hat{\mathcal S}_\epsilon}\max_{j\in [r]}|\left \langle \hat{\bw}_i, \tilde{\bg}_j\right \rangle| \le 3\epsilon,
\end{equation}
where we have used that the term $\norm{\bg_{\rm avg}}$ concentrates around $1/\sqrt{r}$, that $\max_{j\in [r]}\norm{\tilde{\bg}_j}\le 2$, and that \eqref{eq:wlogup} holds. Set $\epsilon'=3\epsilon$. Then, with high probability, $\hat{\mathcal S}_\epsilon\subseteq  \hat{\mathcal S}'_{\epsilon'}$, where $\hat{\mathcal S}'_{\epsilon'}$ is defined as in \eqref{eq:defShatbis}. Consequently, 
\begin{equation}\label{eq:divinit2}
\begin{split}
\sup_{k\ge 3}\sup_{\{\hat{\bw}_i\}_{1\le i\le R} \in \hat{\mathcal S}_{\epsilon}}\sum_{i=1}^r\sum_{j=1}^R\left \langle\bw_i, \hat{\bw}_j\right \rangle^k & \le \sup_{k\ge 3}\sup_{\{\hat{\bw}_i\}_{1\le i\le R} \in \hat{\mathcal S}'_{\epsilon'}}\sum_{i=1}^r\sum_{j=1}^R\left \langle\bw_i, \hat{\bw}_j\right \rangle^k \\
&\le \sup_{k\ge 3}\sup_{\{\hat{\bw}_i\}_{1\le i\le R} \in \hat{\mathcal S}'_{\epsilon'}}\sum_{i=1}^r\sum_{j=1}^R\left \langle\bg_i, \hat{\bw}_j\right \rangle^k \\
&+ \sup_{k\ge 3}\sup_{\{\hat{\bw}_i\}_{1\le i\le R} \in \hat{\mathcal S}'_{\epsilon'}}\left|\sum_{i=1}^r\sum_{j=1}^R\left \langle\bw_i, \hat{\bw}_j\right \rangle^k-\sum_{i=1}^r\sum_{j=1}^R\left \langle\bg_i, \hat{\bw}_j\right \rangle^k\right|.
\end{split}
\end{equation}
By using Lemma \ref{lemma:corrweights_Gall}, we have that, with high probability, the first term in the RHS of \eqref{eq:divinit2} is $R\cdot o(1)$. The rest of the proof consists in showing that, with high probability, the second term in the RHS of \eqref{eq:divinit2} is also $R\cdot o(1)$. 

Consider the set ${\sf S}'_\epsilon(\bg_1, \ldots, \bg_r)$ defined as in \eqref{eq:defSprimeeps}. Then, it is easy to see that
\begin{equation}\label{eq:dec3}
\begin{split}
\sup_{k\ge 3}\sup_{\{\hat{\bw}_i\}_{1\le i\le R} \in \hat{\mathcal S}'_{\epsilon'}}&\left|\sum_{i=1}^r\sum_{j=1}^R\left \langle\bw_i, \hat{\bw}_j\right \rangle^k-\sum_{i=1}^r\sum_{j=1}^R\left \langle\bg_i, \hat{\bw}_j\right \rangle^k\right| \\
&\le R\cdot \sup_{k\ge 3}\max_{\bx\in {\sf S}'_{\epsilon'}(\bg_1, \ldots, \bg_r) } \left|\sum_{i=1}^r c_i^k \left(\left\langle\tilde{\bg}_i, \bx\right \rangle^k-\left \langle\bg_i, \bx\right \rangle^k\right)\right|\\
&+R\cdot \sup_{k\ge 3}\max_{\bx\in {\sf S}'_{\epsilon'}(\bg_1, \ldots, \bg_r) }\left|\sum_{i=1}^r \left(c_i^k -1\right)\left \langle\bg_i, \bx\right \rangle^k\right|.
\end{split}
\end{equation}

Let us provide an upper bound on the first term in the RHS of \eqref{eq:dec3}. For any $\bx\in {\sf S}^{d-1}$, we have that 
\begin{equation}\label{eq:112}
\begin{split}
\left \langle\tilde{\bg}_i, \bx\right \rangle^k- \left \langle\bg_i, \bx\right \rangle^k&= \left(\left \langle\bg_i, \bx\right \rangle-\left \langle\bg_{\rm avg}, \bx\right \rangle\right)^k- \left \langle\bg_i, \bx\right \rangle^k \\
&\stackrel{\mathclap{\mbox{\footnotesize (a)}}}{\le} k\norm{\bg_{\rm avg}}\left(|\left \langle\bg_i, \bx\right \rangle|+\norm{\bg_{\rm avg}}\right)^{k-1}\\
&\stackrel{\mathclap{\mbox{\footnotesize (b)}}}{\le} k2^{k-2}\norm{\bg_{\rm avg}}\left(|\left \langle\bg_i, \bx\right \rangle|^{k-1}+\norm{\bg_{\rm avg}}^{k-1}\right),\\
\end{split}
\end{equation}
where in (a) we use Taylor's inequality applied to the function $p(x)=x^k$, and in (b) we use that $(a+b)^k\le 2^{k-1}(a^{k}+b^{k})$ for $a,b\ge 0$. Note that, with high probability, $\max_{i\in [r]}c_i\le 2$. Hence, \eqref{eq:112} immediately implies that, with high probability,
\begin{equation}\label{eq:111}
\max_{\bx\in {\sf S}'_{\epsilon'}(\bg_1, \ldots, \bg_r) }\left|\sum_{i=1}^r c_i^k \left(\left\langle\tilde{\bg}_i, \bx\right \rangle^k-\left \langle\bg_i, \bx\right \rangle^k\right)\right|\le\max_{\bx\in {\sf S}'_{\epsilon'}(\bg_1, \ldots, \bg_r) } k 2^{2k-2}\left(\norm{\bg_{\rm avg}}\sum_{i=1}^r |\left \langle\bg_i, \bx\right \rangle|^{k-1}+r\cdot \norm{\bg_{\rm avg}}^k\right).
\end{equation}
With high probability, the term $\norm{\bg_{\rm avg}}$ concentrates around $1/\sqrt{r}$. Furthermore,
\begin{equation}\label{eq:opnormg}
\sum_{i=1}^r  |\left \langle\bg_i, \bx\right \rangle|^2 = \left \langle\bx, \sum_{i=1}^r\bg_i\bg_i^{\sT} \bx\right \rangle\le \norm{\sum_{i=1}^r\bg_i\bg_i^{\sT}}_{\rm op},
\end{equation}
where the last inequality uses that $\norm{\bx}=1$. Note that $\sum_{i=1}^r\bg_i\bg_i^{\sT}$ is a Wishart matrix. Hence, with high probability, its operator norm concentrates around $(1+ \sqrt{r/d})^2$ \cite{BaiSilverstein}. As $r = o(d^2/\log d)$, we conclude that, with high probability, 
\begin{equation}
\sup_{k\ge 3}\max_{\bx\in {\sf S}'_{\epsilon'}(\bg_1, \ldots, \bg_r) }k 2^{2k-2}\left(\norm{\bg_{\rm avg}}\sum_{i=1}^r |\left \langle\bg_i, \bx\right \rangle|^{k-1}+r\cdot \norm{\bg_{\rm avg}}^k\right)=o(1).
\end{equation}

Let us now provide an upper bound on the second term in the RHS of \eqref{eq:dec3}. With high probability, $\max_{i\in [r]} c_i\le 2$. Hence, with high probability, 
\begin{equation}\label{eq:intpt2}
\begin{split}
\sup_{k\ge 3}\max_{\bx\in {\sf S}'_{\epsilon'}(\bg_1, \ldots, \bg_r) }\left|\sum_{i=1}^r \left(c_i^k -1\right)\left \langle\bg_i, \bx\right \rangle^k\right| &\le \max_{i\in [r]} |c_i -1|\sup_{k\ge 3}\max_{\bx\in {\sf S}'_{\epsilon'}(\bg_1, \ldots, \bg_r) } \sum_{i=1}^r |4\left \langle\bg_i, \bx\right \rangle|^k\\
&\le \max_{i\in [r]} |c_i -1| \max_{\bx\in {\sf S}'_{\epsilon'}(\bg_1, \ldots, \bg_r) }\sum_{i=1}^r |4\left \langle\bg_i, \bx\right \rangle|^3,
\end{split}
\end{equation}
where the last inequality uses that $4\epsilon'\le 1$. Note that, for any $i\in [r]$,
\begin{equation}
\mathbb P\left(|\norm{\tilde{\bg}_i}-\mathbb E\{\norm{\tilde{\bg}_i}\}|\ge t\right)\le e^{-c dt^2},
\end{equation}
for some constant $c$, since the norm of a vector is a Lipschitz function of its components. Consequently, with high probability,
\begin{equation}\label{eq:bdmax}
\max_{i\in [r]} \left|c_i-1\right|\le \frac{2}{\sqrt{r}}+\frac{2}{\sqrt{c}}\sqrt{\frac{\log d}{d}}.
\end{equation}
In order to upper bound $\sum_{i=1}^r |4\left \langle\bg_i, \bx\right \rangle|^3$, we use an argument similar to that of the proof of Lemma \ref{lemma:corrweights_Gbis}. First of all, note that
\begin{equation}\label{eq:aaaval}
\max_{\bx\in {\sf S}'_{\epsilon'}(\bg_1, \ldots, \bg_r) }\sum_{i=1}^r |4\left \langle\bg_i, \bx\right \rangle|^3 = \max_{\bx\in {\sf S}'_{\epsilon'}(\bg_1, \ldots, \bg_r) }\sum_{i=1}^rh_{4\epsilon'}(4|\left \langle\bg_i, \bx\right \rangle|)\le \max_{\bx\in {\sf S}^{d-1} }\sum_{i=1}^rh_{4\epsilon'}(4|\left \langle\bg_i, \bx\right \rangle|),
\end{equation}
where $h_{4\epsilon'}$ is defined as in \eqref{eq:defhesp}. By Chernoff bound, we obtain that
\begin{equation}\label{eq:cher1}
\mathbb P\left(\sum_{i=1}^rh_{4\epsilon'}(4|\left \langle\bg_i, \bx\right \rangle|) > t\right) \le \exp\left(-\frac{d t}{4^4\epsilon'}\right)\left(\mathbb E\left\{\exp\left(\frac{d h_{4\epsilon'}(4|\left \langle\bg_1, \bx\right \rangle|)}{4^4\epsilon'}\right)\right\}\right)^r.
\end{equation}
As $\left \langle\bg_1, \bx\right \rangle \sim \normal(0, 1/d)$, we have that 
\begin{equation}\label{eq:cher2}
\mathbb E\left\{\exp\left(\frac{d h_{4\epsilon'}(4|\left \langle\bg_1, \bx\right \rangle|)}{4^4\epsilon'}\right)\right\} \le 2\int_{0}^{\sqrt{8\log d}} \frac{1}{\sqrt{2\pi}}\exp\left(\frac{x^3}{4\epsilon'\sqrt{d}}-\frac{x^2}{2}\right)\,{\rm d}x + \frac{2\sqrt{2}}{d^2}+2 \exp\left(-\frac{d(\epsilon')^2}{4}\right).
\end{equation}
Furthermore, after some calculations, we obtain that
\begin{equation}\label{eq:cher3}
\begin{split}
2\int_{0}^{\sqrt{8\log d}} \frac{1}{\sqrt{2\pi}}\exp\left(\frac{x^3}{4\epsilon'\sqrt{d}}-\frac{x^2}{2}\right)\,{\rm d}x&\le 2\int_{0}^{\sqrt{8\log d}} \frac{1}{\sqrt{2\pi}}\exp\left(x^2\left(-\frac{1}{2}+\frac{\sqrt{8\log d}}{4\epsilon'\sqrt{d}}\right)\right)\,{\rm d}x\\
&\le 1+\frac{c'}{\epsilon'}\sqrt{\frac{\log d}{d}},
\end{split}
\end{equation}
for some constant $c'$. By combining \eqref{eq:cher1}, \eqref{eq:cher2}, and \eqref{eq:cher3} with the fact that $\log(1+x)\le x$ for any $x\ge 0$, we conclude that
\begin{equation}
\mathbb P\left(\sum_{i=1}^rh_{4\epsilon'}(4|\left \langle\bg_i, \bx\right \rangle|) > t\right) \le\exp\left(-\frac{d t}{4^4\epsilon'}+r\left(\frac{c'}{\epsilon'}\sqrt{\frac{\log d}{d}}+ \frac{2\sqrt{2}}{d^2}+2 \exp\left(-\frac{d(\epsilon')^2}{4}\right)\right)\right).
\end{equation}
This bound holds for a fixed $\bx\in {\sf S}^{d-1}$. In order to obtain a bound which is uniform over $\bx$, let $N^d(\alpha)$ be an $\alpha$-net of ${\sf S}^{d-1}$. Then, with high probability, for any  $\bx \in {\sf S}^{d-1}$ and $\bx'\in N^d(\alpha)$ s.t. $\norm{\bx-\bx'}\le \alpha$,
\begin{equation}\label{eq:boundpoint_nn}
\max_{\bx\in {\sf S}^{d-1}}\sum_{i=1}^rh_{4\epsilon'}(4|\left \langle\bg_i, \bx\right \rangle|) \le \max_{\bx\in N^d(\alpha)}\sum_{i=1}^rh_{4\epsilon'}(4|\left \langle\bg_i, \bx\right \rangle|) + 8C\alpha r,
\end{equation}
where $C$ is a constant that upper bounds the Lipschitz constant of $h_{4\epsilon'}$. Furthermore, by using \eqref{eq:boundpoint_nn} and a union bound over the points of the $\alpha$-net, we obtain that
\begin{equation}\label{eq:epsnet_newbound}
\begin{split}
\mathbb P(\max_{\bx\in N^d(\alpha)}&\sum_{i=1}^rh_{4\epsilon'}(4|\left \langle\bg_i, \bx\right \rangle|) > t)\\
&\le \exp\left(-\frac{d t}{4^4\epsilon'}+r\left(\frac{c'}{\epsilon'}\sqrt{\frac{\log d}{d}}+ \frac{2\sqrt{2}}{d^2}+2 \exp\left(-\frac{d(\epsilon')^2}{4}\right)\right)\right) \left(1+\frac{2}{\alpha}\right)^d,
\end{split}
\end{equation}
where we have used the fact that there exists an $\alpha$-net of ${\sf S}^{d-1}$ that contains at most $(1+2/\alpha)^d$ points, see Lemma 5.2 of \cite{vershynin2010introduction}. Pick $\alpha = 1/r^2$ and
\begin{equation*}
t = 2^9\,c' \max\left(1, \frac{r\,\sqrt{\log d}}{d^{3/2}}\right).
\end{equation*}
Then, \eqref{eq:aaaval}, \eqref{eq:boundpoint_nn} and \eqref{eq:epsnet_newbound} imply that, with high probability, 
\begin{equation}
\max_{\bx\in {\sf S}'_{\epsilon'}(\bg_1, \ldots, \bg_r)}\sum_{i=1}^r |4\left \langle\bg_i, \bx\right \rangle|^k \le t+1.
\end{equation}
By using \eqref{eq:bdmax} and that $r =o(d^2/\log d)$, we obtain that the RHS of \eqref{eq:intpt2} is $o(1)$. Consequently, the RHS of \eqref{eq:dec3} is $R\cdot o(1)$, which concludes the proof.
\end{proof}

Eventually, we prove our main result on the generalization error in the setting with random weights. 

\begin{proof}[Proof of Theorem \ref{th:lowergen_G}]
The procedure is similar to that used to prove Theorem \ref{th:lowergen} in Appendix \ref{app:mainproof}. As $\sigma \in L^2(\mathbb R, e^{-x^2/2})$ and the weights $\{\bw_i\}_{1\le i\le r}$ and $\{\hat{\bw}_i\}_{1\le i \le R}$ have unit norm, \eqref{eq:first} and \eqref{eq:secfisteq} hold. Therefore,
\begin{equation}\label{eq:step1_G}
{\mathbb E}\big\{|y(\bx)-\hat{y}(\bx)|^2\big\}\ge\sum_{k\ge 3} \hat{\sigma}_k^2 \left(\norm{\sum_{i=1}^r\bw_i^{\otimes k}}_F^2-2\sum_{i=1}^r\sum_{j=1}^R \left \langle \bw_i, \hat{\bw}_j\right \rangle^k\right).
\end{equation}
Furthermore, 
\begin{equation*}
\norm{\sum_{i=1}^r\bw_i^{\otimes k}}_F^2 = \sum_{i, j=1}^r \left \langle \bw_i, \bw_j\right \rangle^k = r+\sum_{i\neq j} \left \langle \bw_i, \bw_j\right \rangle^k.
\end{equation*}
Let us now show that, with high probability,
\begin{equation}\label{eq:sumcorrdiff}
\sup_{k\ge 3}\sum_{i\neq j} \left \langle \bw_i, \bw_j\right \rangle^k = o(r).
\end{equation}
We start by proving that
\begin{equation}\label{eq:sumcorrdiff_2}
\sup_{k\ge 4}\sum_{i\neq j} \left \langle \bw_i, \bw_j\right \rangle^k = o(r).
\end{equation}
First, note that
\begin{equation}
\begin{split}
\sup_{k\ge 4}\sum_{i\neq j} \left \langle \bw_i, \bw_j\right \rangle^k \le\sup_{k\ge 4} r^2 \max_{i\neq j} |\left \langle \bw_i, \bw_j\right \rangle|^k\le r^2 \max_{i\neq j} |\left \langle \bw_i, \bw_j\right \rangle|^4 \le r^2 \left(\max_{i\neq j} |\left \langle \bw_i, \bw_j\right \rangle|\right)^4.
\end{split}
\end{equation}
Recall the definition \eqref{eq:defrandomw} of the weights $\{\bw_i\}_{1\le i\le r}$ and the definitions in \eqref{eq:defrandomwextra}. Then,
\begin{equation}
\begin{split}
|\left \langle \bw_i, \bw_j\right \rangle| &= c_i c_j |\left \langle\bg_i-\bg_{\rm avg},\bg_j-\bg_{\rm avg}\right \rangle| \\
&\le c_i c_j |\left \langle\bg_i,\bg_j\right \rangle|  +(c_i+c_j)\norm{\bg_{\rm avg}} +c_ic_j\norm{\bg_{\rm avg}}^2.
\end{split}
\end{equation}
With high probability, the term $\norm{\bg_{\rm avg}}$ concentrates around $1/\sqrt{r}$. Furthermore, with high probability, we have that
\begin{equation}\label{eq:boundasc}
\max_{i\in [r]} c_i\le 2,
\end{equation}
\begin{equation}\label{eq:boundasg}
\max_{i\neq j}|\left \langle\bg_i,\bg_j\right \rangle| \le C\sqrt{\frac{\log d}{d}},
\end{equation}
for some constant $C$. Hence, with high probability,
\begin{equation}
\max_{i\neq j} |\left \langle \bw_i, \bw_j\right \rangle| \le 4C\sqrt{\frac{\log d}{d}} +\frac{8}{\sqrt{r}}. 
\end{equation}
As $r=o(d^2/(\log d)^2)$, \eqref{eq:sumcorrdiff} immediately follows.

It remains to deal with the case $k=3$. Note that
\begin{equation}\label{eq:sumdiv1}
\sum_{i\neq j} \left \langle \bw_i, \bw_j\right \rangle^3 = A + B + C,
\end{equation}
where
\begin{equation}
\begin{split}
A & = \sum_{i\neq j} \left \langle \bw_i, \bw_j\right \rangle^3-\bar{c}^6\sum_{i\neq j} \left \langle \tilde{\bg}_i, \tilde{\bg}_j\right \rangle^3,\\
B & = \bar{c}^6\left(\sum_{i\neq j} \left \langle \tilde{\bg}_i, \tilde{\bg}_j\right \rangle^3-\sum_{i\neq j} \left \langle \bg_i, \bg_j\right \rangle^3\right),\\
C & =\bar{c}^6\sum_{i\neq j} \left \langle \bg_i, \bg_j\right \rangle^3. \\
\end{split}
\end{equation}
Let us provide an upper bound on the term $A$ of the RHS of \eqref{eq:sumdiv1}: 
\begin{equation}\label{eq:term1c}
\begin{split}
\sum_{i\neq j} \left \langle \bw_i, \bw_j\right \rangle^3-\bar{c}^6\sum_{i\neq j} \left \langle \tilde{\bg}_i, \tilde{\bg}_j\right \rangle^3 &= \sum_{i\neq j} \left(c_i^3 c_j^3-\bar{c}^6\right)\left \langle \tilde{\bg}_i, \tilde{\bg}_j\right \rangle^3 \\
&\le \sqrt{\sum_{i\neq j} \left(c_i^3 c_j^3-\bar{c}^6\right)^2}\sqrt{\sum_{i\neq j}\left \langle \tilde{\bg}_i, \tilde{\bg}_j\right \rangle^6},
\end{split}
\end{equation}
where we have used Cauchy-Schwarz inequality. Furthermore, we have that
\begin{equation}\label{eq:div1}
\begin{split}
|c_i^3 c_j^3-\bar{c}^6| &\le |c_i^3 c_j^3-c_i^3\bar{c}^3|+|c_i^3\bar{c}^3-\bar{c}^6|\\
&\le c_i^3\left(c_j^2+c_j \bar{c}+\bar{c}^2\right)| c_j-\bar{c}|+\bar{c}^3\left(c_i^2+c_i \bar{c}+\bar{c}^2\right)| c_i-\bar{c}|.
\end{split}
\end{equation}
By using \eqref{eq:div1}, \eqref{eq:boundasc}, and \eqref{eq:bdmax}, we obtain that, with high probability,
\begin{equation}\label{eq:bdas1}
\sqrt{\sum_{i\neq j} \left(c_i^3 c_j^3-\bar{c}^6\right)^2} \le  C_1\left(\sqrt{r}+r\sqrt{\frac{\log d}{d}}\right),
\end{equation}
for some constant $C_1$. Furthermore, by using \eqref{eq:boundasg}, we also obtain that, with high probability,
\begin{equation}\label{eq:bdas2}
\sqrt{\sum_{i\neq j}\left \langle \tilde{\bg}_i, \tilde{\bg}_j\right \rangle^6}\le C^3\frac{r(\log d)^{3/2}}{d^{3/2}}.
\end{equation}
As $r=o(d^2/(\log d)^2)$ , by combining \eqref{eq:bdas1} with \eqref{eq:bdas2}, we conclude that $A=o(r)$.

Let us provide an upper bound on the term $B$ of the RHS of \eqref{eq:sumdiv1}. First, note that
\begin{equation}
\begin{split}
\left \langle \tilde{\bg}_i, \tilde{\bg}_j\right\rangle^3 - \left \langle \bg_i, \bg_j\right\rangle^3 &= \left(\left \langle \bg_i, \bg_j\right\rangle-\left \langle \bg_i+ \bg_j, \bg_{\rm avg}\right\rangle + \norm{\bg_{\rm avg}}^2\right)^3 - \left \langle \bg_i, \bg_j\right\rangle^3\\
&\stackrel{\mathclap{\mbox{\footnotesize (a)}}}{\le} \left(\left \langle \bg_i, \bg_j\right\rangle +3\norm{\bg_{\rm avg}}\right)^3 - \left \langle \bg_i, \bg_j\right\rangle^3\\
&\stackrel{\mathclap{\mbox{\footnotesize (b)}}}{\le} 3\norm{\bg_{\rm avg}}\left(|\left \langle \bg_i, \bg_j\right\rangle| +3\norm{\bg_{\rm avg}}\right)^2 \\
&\stackrel{\mathclap{\mbox{\footnotesize (c)}}}{\le} 6\norm{\bg_{\rm avg}}\left(|\left \langle \bg_i, \bg_j\right\rangle|^2 +9\norm{\bg_{\rm avg}}^2\right) \\
\end{split}
\end{equation}
where in (a) we use that the function $p(x)=x^3$ is increasing and that $\max_{i\in [r]}\norm{\bg_i}\le 2$ with high probability, in (b) we use Taylor's inequality applied to the function $p(x)=x^k$, and in (c) we use that $(a+b)^2 \le 2(a^2 + b^2)$ for any $a, b\ge 0$. By summing over $i\neq j$, we have that 
\begin{equation}\label{eq:eqdiv2}
\sum_{i\neq j}\left \langle \tilde{\bg}_i, \tilde{\bg}_j\right\rangle^3 - \sum_{i\neq j}\left \langle \bg_i, \bg_j\right\rangle^3 \le 6\norm{\bg_{\rm avg}}\sum_{i\neq j}|\left \langle \bg_i, \bg_j\right\rangle|^2 + 18 r^2\norm{\bg_{\rm avg}}^3.
\end{equation}
As $\norm{\bg_{\rm avg}}$ concentrates around $1/\sqrt{r}$, the term $r^2\norm{\bg_{\rm avg}}^3$ is $o(r)$. Furthermore, for any $j\neq i$,
\begin{equation}
\sum_{i=1}^r|\left \langle \bg_i, \bg_j\right\rangle|^2 = \left\langle\bg_j,\sum_{i=1}^r \bg_i \bg_i^{\sT} \bg_j \right\rangle\le \norm{\bg_j}^2\norm{\sum_{i=1}^r \bg_i \bg_i^{\sT}}_{\rm op}.
\end{equation}
As $\sum_{i=1}^r \bg_i \bg_i^{\sT}$ is a Wishart matrix, with high probability, its operator norm concentrates around $(1 +\sqrt{r/d})^2$ \cite{BaiSilverstein}. As $r=o(d^2)$, we conclude that $B=o(r)$.

Let us provide an upper bound on the term $C$ of the RHS of \eqref{eq:sumdiv1}. To do so, we upper bound the second moment:
\begin{equation}
\begin{split}
\mathbb E\left\{\left(\frac{1}{r}\sum_{i\neq j} \left \langle \bg_i, \bg_j\right \rangle^3\right)^2 \right\} &= \frac{1}{r^2}\sum_{i\neq j}\sum_{k\neq \ell}\mathbb E\left\{\left \langle \bg_i, \bg_j\right \rangle^3\left \langle \bg_k, \bg_\ell\right \rangle^3\right\}\\
&=\frac{2}{r^2}\sum_{i\neq j}\mathbb E\left\{\left \langle \bg_i, \bg_j\right \rangle^6\right\}\\
&\le2\max_{i\neq j}\mathbb E\left\{\left \langle \bg_i, \bg_j\right \rangle^6\right\}.\\
\end{split}
\end{equation}
Hence, after some simple calculations, we deduce that, with high probability,
\begin{equation}\label{eq:o1}
\frac{1}{r}\sum_{i\neq j} \left \langle \bg_i, \bg_j\right \rangle^3 = o(1).
\end{equation}
As $A=o(r)$, $B=o(r)$ and $C=o(r)$,  the RHS of \eqref{eq:sumdiv1} is also $o(r)$ with high probability. As a result, \eqref{eq:sumcorrdiff} holds with high probability.

By combining \eqref{eq:sumcorrdiff} with \eqref{eq:step1_G} and with the result of Lemma \ref{lemma:corrweights_G}, we conclude that, with high probability, for a sequence of vanishing constants $\eta(r, d)=o(1)$,
\begin{equation}
\sup_{\{\hat{\bw}_i\}_{1\le i\le R} \in \hat{\mathcal S}_\epsilon}{\mathbb E}\big\{|y(\bx)-\hat{y}(\bx)|^2\big\}\ge \sum_{k\ge 3} \hat{\sigma}_k^2\norm{\sum_{i=1}^r\bw_i^{\otimes k}}_F^2 \left(1-\frac{R}{r}\eta(r, d)\right).
\end{equation}
By following the same passages as those of the third step of the proof of Theorem \ref{th:lowergen} and by using that $\hat{\sigma}_2 = 0$, we also have that 
\begin{equation}\label{eq:third_G}
\begin{split}
\sum_{k\ge 3} \hat{\sigma}_k^2\norm{\sum_{i=1}^r\bw_i^{\otimes k}}_F^2&\ge \min_{a\in \mathbb R}{\mathbb E}\left\{\left|y(\bx)-a\right|^2\right\}-2\hat{\sigma}_1^2 \norm{\sum_{i=1}^r \bw_i}^2\\
&= \Var\left\{y(\bx)\right\}-2\hat{\sigma}_1^2 \norm{\sum_{i=1}^r \bw_i}^2,
\end{split}
\end{equation}
which implies that
\begin{equation}\label{eq:zzz1}
\sup_{\{\hat{\bw}_i\}_{1\le i\le R} \in \hat{\mathcal S}_\epsilon}{\mathbb E}\big\{|y(\bx)-\hat{y}(\bx)|^2\big\}\ge\left(\Var\left\{y(\bx)\right\}-2\hat{\sigma}_1^2 \norm{\sum_{i=1}^r \bw_i}^2\right)\left(1-\frac{R}{r}\eta(r, d)\right).
\end{equation}

It remains to upper bound the term $\norm{\sum_{i=1}^r \bw_i}^2$. We do so by computing its expected value:
\begin{equation}
\begin{split}
\mathbb E\left\{\norm{\sum_{i=1}^r\bw_i}^2\right\} &= \mathbb E\left\{\norm{\sum_{i=1}^r \left(c_i-\bar{c}\right)\tilde{\bg}_i}^2\right\}\\
&= r \mathbb E\left\{ \left(c_1-\bar{c}\right)^2\norm{\tilde{\bg}_1}^2\right\} + r(r-1)\mathbb E\left\{ \left(c_1-\bar{c}\right)\left(c_2-\bar{c}\right)\left\langle\tilde{\bg}_1, \tilde{\bg}_2\right\rangle\right\}\\
\end{split}
\end{equation}
where in the first equality we use that $\sum_{i=1}^r \tilde{\bg}_i = 0$. After some calculations, we have that 
\begin{equation}
\mathbb E\left\{ \left(c_1-\bar{c}\right)^2\norm{\tilde{\bg}_1}^2\right\}\le \frac{1}{d}.
\end{equation}
Furthermore, by applying Stein's lemma for correlated random variables, we have that
\begin{equation}\label{eq:sumlast1}
\begin{split}
\mathbb E\left\{ \left(c_1-\bar{c}\right)\left(c_2-\bar{c}\right)\left\langle\tilde{\bg}_1, \tilde{\bg}_2\right\rangle\right\} &= -\frac{1}{r}\mathbb E\left\{ \left(c_1-\bar{c}\right)\left(c_2-\bar{c}\right)\right\} \\
&+\frac{1}{d}\left(1-\frac{1}{r}\right)\mathbb E\left\{ \left(c_2-\bar{c}\right)\left\langle-\frac{\tilde{\bg}_1}{\norm{\tilde{\bg}_1}^3}, \tilde{\bg}_2\right\rangle\right\}\\
&+\frac{1}{d}\left(-\frac{1}{r}\right)\mathbb E\left\{ \left(c_1-\bar{c}\right)\left\langle-\frac{\tilde{\bg}_2}{\norm{\tilde{\bg}_2}^3}, \tilde{\bg}_2\right\rangle\right\}.\\
\end{split}
\end{equation}
We upper bound the first term in the RHS of \eqref{eq:sumlast1} as
\begin{equation}
\frac{1}{r}|\mathbb E\left\{ \left(c_1-\bar{c}\right)\left(c_2-\bar{c}\right)\right\}|\le \frac{1}{r\cdot d}.
\end{equation}
By applying Cauchy-Schwarz inequality, we upper bound the second term in the RHS of \eqref{eq:sumlast1} as
\begin{equation}\label{eq:steineq2}
\frac{1}{d}\left(1-\frac{1}{r}\right)\left|\mathbb E\left\{ \left(c_2-\bar{c}\right)\left\langle-\frac{\tilde{\bg}_1}{\norm{\tilde{\bg}_1}^3}, \tilde{\bg}_2\right\rangle\right\}\right|\le \frac{1}{d}\sqrt{\mathbb E\left\{\left(c_2-\bar{c}\right)^2\right\}}\sqrt{\mathbb E\left\{\frac{\left\langle\tilde{\bg}_1, \tilde{\bg}_2\right\rangle^2}{\norm{\tilde{\bg}_1}^6}\right\}}.
\end{equation}
Note that
\begin{equation}
\begin{split}
\left\langle\tilde{\bg}_1, \tilde{\bg}_2\right\rangle^2 &= \left(\left\langle\bg_1, \bg_2\right\rangle - \left\langle\bg_1+ \bg_2,\bg_{\rm avg}\right\rangle + \norm{\bg_{\rm avg}}^2\right)^2 \\
&\le 2\left(\left\langle\bg_1, \bg_2\right\rangle - \left\langle\bg_1+ \bg_2,\bg_{\rm avg}\right\rangle\right)^2 + 2\norm{\bg_{\rm avg}}^4 \\
&\le 4\left\langle\bg_1, \bg_2\right\rangle^2 + 4\left\langle\bg_1+ \bg_2,\bg_{\rm avg}\right\rangle^2 +2 \norm{\bg_{\rm avg}}^4 \\
&\le 4\left\langle\bg_1, \bg_2\right\rangle^2 + 4\left(\norm{\bg_1}+ \norm{\bg_2}\right)^2\norm{\bg_{\rm avg}}^2 +2 \norm{\bg_{\rm avg}}^4, \\
\end{split}
\end{equation}
which implies that the RHS of \eqref{eq:steineq2} is at most $4/d+18/r$.

We upper bound the third term in the RHS of \eqref{eq:sumlast1} as
\begin{equation}
\frac{1}{d\cdot r}\left|\mathbb E\left\{ \left(c_1-\bar{c}\right)\left\langle-\frac{\tilde{\bg}_2}{\norm{\tilde{\bg}_2}^3}, \tilde{\bg}_2\right\rangle\right\}\right|\le \frac{1}{r\cdot d^{3/2}}.
\end{equation}
By using that $r=o(d^2)$, we deduce that
\begin{equation}
\mathbb E\left\{\norm{\sum_{i=1}^r\bw_i}^2\right\}  = o(r) = r \cdot \eta(r, d).
\end{equation}
Hence, by Markov's inequality, we conclude that, with high probability,
\begin{equation}
\norm{\sum_{i=1}^r\bw_i}^2 = o(r),
\end{equation}
which, combined with \eqref{eq:zzz1}, implies the desired result. 
\end{proof}

%----------------------------------------------------
\section{Learning a Neural Network and Tensor Decomposition: Proof of Theorem \ref{th:comp}} \label{sec:proofhard}
%----------------------------------------------------

\begin{proof}
We start by proving the first claim. Assume that the thesis is false. Then, there exists an algorithm $\mathcal A$ that, given $\{(\bx_j, y(\bx_j))\}_{1\le j\le n}$, has polynomial complexity and outputs $\{\hat{\bw}_{i}\}_{1\le i \le R}$ with unit norm s.t. $|\left \langle \bw_i, \hat{\bw}_j\right \rangle| \ge \epsilon$ for some $i \in [r]$ and $j\in [R]$. Note that, as $\mathcal A$ has polynomial complexity, we can assume without loss of generality that $n$ is bounded by a polynomial in $d$.

If $\sigma$ is a polynomial with ${\rm deg}(\sigma)\le \ell$ and ${\rm par}(\sigma)={\rm par}(\ell)$, then it can be written as
\begin{equation}\label{eq:defsigma}
\sigma(z) = \sum_{\substack{k\in [\ell]\\ {\rm par}(k) = {\rm par}(\ell)}} c_k z^k,
\end{equation}
for some choice of the coefficients $\{c_k\}_{k\in [\ell],\, {\rm par}(k) = {\rm par}(\ell)}$.
By definition of tensor, we have that
\begin{equation}\label{eq:tensred}
\begin{split}
\sum_{j=1}^d  T^{(\ell)}(j, j, j_3, \ldots, j_\ell) 
&= \sum_{j=1}^d \sum_{i=1}^r w_i(j)\cdot w_i(j) \prod_{m=3}^\ell w_i(j_m)\\
&= \sum_{i=1}^r \norm{\bw_i}^2 \prod_{m=3}^\ell w_i(j_m) = T^{(\ell-2)}(j_3, \ldots, j_\ell),
\end{split}
\end{equation}
where the last equality follows from the fact that the weights $\{\bw_i\}_{1\le i \le r}$ have unit norm. Consequently, given the tensor $\bT^{(\ell)}$, we can construct with polynomial complexity the tensor $\bT^{(k)}$ for any $k\in [\ell]$ such that ${\rm par}(k) = {\rm par}(\ell)$. This implies that, given $\bx_j$, we can construct with polynomial complexity the following quantity:
\begin{equation}\label{eq:quantity}
\sum_{\substack{k\in [\ell]\\ {\rm par}(k) = {\rm par}(\ell)}} c_k \left \langle \bT^{(k)}, \bx_j^{\otimes k}\right \rangle. 
\end{equation} 
By applying \eqref{eq:tensscal} and \eqref{eq:defsigma}, we obtain that the quantity in \eqref{eq:quantity} equals $y(\bx_j)$. Consequently, we can construct the set $\{(\bx_j, y(\bx_j))\}_{1\le j\le n}$ with polynomial complexity. By applying the algorithm $\mathcal A$ with input $\{(\bx_j, y(\bx_j))\}_{1\le j\le n}$, we obtain with polynomial complexity the estimates $\{\hat{\bw}_{i}\}_{1\le i \le R}$ with unit norm s.t. $|\left \langle \bw_i, \hat{\bw}_j\right \rangle| \ge \epsilon$ for some $i \in [r]$ and $j\in [R]$. As a result, there exists an algorithm that, given the tensor $\bT^{(\ell)}$, has polynomial complexity and outputs $\{\hat{\bw}_{i}\}_{1\le i \le R}$ with unit norm s.t. $|\left \langle \bw_i, \hat{\bw}_j\right \rangle| \ge \epsilon$ for some $i \in [r]$ and $j\in [R]$. Hence, given the tensor $\bT^{(\ell)}$, the problem of learning $\{\bw_{i}\}_{1\le i \le r}$ is not $\epsilon$-hard, which violates the hypothesis and concludes the proof of the first claim.

The proof of the second claim is similar. Suppose there exists an algorithm $\mathcal A'$ that, given $\{(\bx_j, y(\bx_j))\}_{1\le j\le n}$, has polynomial complexity and outputs $\{\hat{\bw}_{i}\}_{1\le i \le R}$ with unit norm s.t. $|\left \langle \bw_i, \hat{\bw}_j\right \rangle| \ge \epsilon$ for some $i \in [r]$ and $j\in [R]$. If $\sigma$ is a polynomial with ${\rm deg}(\sigma)\le \ell+1$, then it can be written as
\begin{equation}\label{eq:defsigmanew}
\sigma(z) = \sum_{k=0}^{\ell+1} c'_k z^k,
\end{equation}
for some choice of the coefficients $\{c'_k\}_{0\le k\le \ell+1}$. By using \eqref{eq:tensred}, given the tensors $\bT^{(\ell)}$ and $\bT^{(\ell+1)}$, we can construct with polynomial complexity the tensor $\bT^{(k)}$ for any $k\in [\ell+1]$. This implies that, given $\bx_j$, we can construct with polynomial complexity the quantity
\begin{equation}\label{eq:quantitynew}
 \sum_{k=0}^{\ell+1} c'_k \left \langle \bT^{(k)}, \bx_j^{\otimes k}\right \rangle,
\end{equation} 
that is equal to $y(\bx_j)$. By using the algorithm $\mathcal A'$, we have found an algorithm that, given the tensors $\bT^{(\ell)}$ and $\bT^{(\ell+1)}$, has polynomial complexity and outputs $\{\hat{\bw}_{i}\}_{1\le i \le R}$ with unit norm s.t. $|\left \langle \bw_i, \hat{\bw}_j\right \rangle| \ge \epsilon$ for some $i \in [r]$ and $j\in [R]$. 
Hence, the hypothesis is violated and the proof is complete.
\end{proof}

%----------------------------------------------------
\section{Learning a Neural Network and Tensor Decomposition -- Noisy Case} \label{subsec:genhard}
%----------------------------------------------------

Let us consider a slightly different model of two-layer neural network with an error term $E(\bx)$, where the output $y_{\rm noisy}(\bx)$ is given by 
\begin{equation}\label{eq:defynoise}
y_{\rm noisy}(\bx) = y(\bx)+E(\bx)= \sum_{i=1}^r \sigma(\left \langle \bx, \bw_i\right \rangle) + E(\bx).
\end{equation}
For $\delta\ge 0$, define
\begin{equation}
\mathcal S_\delta\subseteq\{\{\bw_i\}_{1\le i\le r} : \norm{\bw_i}=1 \,  \,\,\forall\, i\in [r],\,\,|\left \langle \bw_i, \bw_j\right \rangle | \le \delta\, \,\,\forall\, i\neq j \in [r]\}.
\end{equation}
We now state the reduction from tensor decomposition to the problem of learning the weights of a two-layer neural network with noisy output and activation function which is a polynomial with degree \emph{larger} than the order of the tensor.

\begin{theorem}[Learning a Neural Network and Tensor Decomposition -- Noisy Case]\label{th:compnoise}
Fix integers $\ell\ge 3$, $p\in [\ell+1]$, and $m\in [\lfloor\sfrac{\ell}{(p-1)}\rfloor]$. Assume also that $p$ is even.  Let $\sigma$ be an even positive polynomial that can be written as
\begin{equation}\label{eq:sigmahp}
\sigma(z) = \sum_{k=m}^{\lfloor\ell/(p-1)\rfloor} c_{k} z^{p (\ell-(p-1)k)},
\end{equation}
for some choice of the positive coefficients $\{c_k\}_{m\le k \le \lfloor\ell/(p-1)\rfloor}$. Let $y_{\rm noisy}(\bx)$ be defined in \eqref{eq:defynoise}. For $\bx_1, \ldots, \bx_n\in \mathbb R^d$, let $\mathcal P_{\rm noisy}(\bx_1, \ldots, \bx_n)$ be the problem of learning $\{\bw_i\}_{1\le i\le r}\in\mathcal S_\delta$ given as input $\{\bx_j\}_{1\le j \le n}$ and $\{y_{\rm noisy}(\bx_j)\}_{1\le j \le n}$. Assume that, given as input the tensor $\bT^{(\ell)}$ defined in \eqref{eq:deftns}, the problem of learning $\{\bw_i\}_{1\le i\le r}\in\mathcal S_\delta$ is $\epsilon$-hard in the sense of Definition \ref{def:hardeps} for some $\epsilon>0$. Then, there exists a choice of the error term $E(\bx)$ with
\begin{equation}\label{eq:err}
|E(\bx)|\le (\delta^m \cdot r)^{p-1} \cdot y(\bx), \quad\quad \forall \,\bx\in\mathbb R^d,
\end{equation}
such that, for any $\bx_1, \ldots, \bx_n\in \mathbb R^d$, the problem $\mathcal P_{\rm noisy}(\bx_1, \ldots, \bx_n)$ is $\epsilon$-hard in the sense of Definition \ref{def:hardeps}.
\end{theorem}

As an example, set $p=2$ and $m=3$. Then, the activation function $\sigma$ has the form
\begin{equation}
\sigma(z) = a_0 + a_1 z^2 + a_2 z^4 + \cdots + a_{\ell-3} z^{2(\ell-3)},
\end{equation}
for some choice of the positive coefficients $a_0, \ldots, a_{\ell-3}$.
Furthermore, the error term $E(\bx)$ is negligible with respect to the signal $y(\bx)$ as long as $\delta^3 \cdot r$ is also negligible. Thus, if $\delta$ is at most of order $1/\sqrt{d}$, Theorem \ref{th:compnoise} holds for $r$ as large as $d^{3/2}$.

\begin{proof}[Proof of Theorem \ref{th:compnoise}]
The proof follows the lines of that of Theorem \ref{th:comp}. Assume that the thesis is false. Then, there exists an algorithm $\mathcal A$ that, given $\{(\bx_j, y_{\rm noisy}(\bx_j))\}_{1\le j\le n}$, has polynomial complexity and outputs $\{\hat{\bw}_{i}\}_{1\le i \le R}$ with unit norm s.t. $|\left \langle \bw_i, \hat{\bw}_j\right \rangle| \ge \epsilon$ for some $i \in [r]$ and $j\in [R]$. Without loss of generality, we can assume that $n$ is bounded by a polynomial in $d$.

Let $k\in \{m, m+1, \ldots, \lfloor\sfrac{\ell}{(p-1)}\rfloor\}$ and define the tensor $\bT_0^{(k)}$ of order $p(\ell-(p-1)k)$ as follows:
\begin{equation}
\begin{split}
T_0^{(k)}(&j^{(1)}_{(p-1)k+1}, \ldots, j^{(1)}_{\ell}, j^{(2)}_{(p-1)k+1},\ldots,  j^{(2)}_{\ell}, \ldots, j^{(p)}_{(p-1)k+1}, \ldots, j^{(p)}_{\ell})\\
& =\sum_{\{j_a^{(b, c)}\}}\prod_{q=1}^p T^{(\ell)}(j_1^{(1, q)}, \ldots, j_k^{(1, q)}, \ldots, j_1^{(q-1, q)}, \ldots, j_k^{(q-1, q)}, j^{(q)}_{(p-1)k+1}, \ldots, j^{(q)}_{\ell},\\
&\hspace{4em} j_1^{(q, q+1)}, \ldots, j_k^{(q, q+1)}, \ldots, j_1^{(q, p)}, \ldots, j_k^{(q, p)}),
\end{split}
\end{equation}
where the first sum is over all indices $j_a^{(b, c)}$, where $a\in [k]$, $b < c$ and $b, c\in [p]$. In words, in order to obtain $\bT_0^{(k)}$, we multiply $p$ copies of the tensor $T^{(\ell)}$, each pair of copies shares $k$ indices, and we perform the summation over those shared indices. By definition of tensor, we have that
\begin{equation}\label{eq:tens1}
\begin{split}
T_0^{(k)}(j^{(1)}_{(p-1)k+1}, \ldots, j^{(1)}_{\ell}, j^{(2)}_{(p-1)k+1}, \ldots, 
& j^{(2)}_{\ell}, \ldots, j^{(p)}_{(p-1)k+1}, \ldots, j^{(p)}_{\ell})\\
&= \sum_{i_1, \ldots, i_p\in [r]} \prod_{\substack{b, c\in [p]\\b<c}}\left \langle \bw_{i_{b}}, \bw_{i_c} \right \rangle^k \prod_{q = 1}^p \prod_{s=(p-1)k+1}^\ell w_{i_q}(j^{(q)}_{s}).
\end{split}
\end{equation}

Note that, given the tensor $\bT^{(\ell)}$, we can construct with polynomial complexity the tensor $\bT_0^{(k)}$ for any $k\in \{m, m+1, \ldots, \lfloor\sfrac{\ell}{(p-1)}\rfloor\}$. Hence, given $\bx_j$, we can also construct with polynomial complexity the following quantity:
\begin{equation}\label{eq:q0}
\sum_{k=m}^{\lfloor \ell/(p-1)\rfloor}c_k\left \langle \bT_0^{(k)}, \bx_j^{\otimes p(\ell-(p-1)k)} \right \rangle,
\end{equation}
which, by using \eqref{eq:tens1}, can be rewritten as 
\begin{equation}\label{eq:q1}
\sum_{k=m}^{\lfloor \ell/(p-1)\rfloor}  c_k \sum_{i_1, \ldots, i_p\in [r]} \prod_{\substack{b, c\in [p]\\b<c}}\left \langle \bw_{i_{b}}, \bw_{i_c} \right \rangle^k \prod_{q = 1}^p\left \langle \bw_{i_q}, \bx_j \right \rangle^{\ell-(p-1)k} .
\end{equation}
Let $\mathcal D$ be the set of $p$-tuples $(i_1, \ldots, i_p)$ whose components are all equal, i.e., 
\begin{equation}
\mathcal D = \{(i_1, \ldots, i_p) : i_b =i_c \quad \forall\,b, c\in [p]\},
\end{equation}
 and let $\mathcal D^{\rm c}$ be its complement, i.e.,
 \begin{equation}
\mathcal D^{\rm c} = [r]^p\setminus \mathcal D.
\end{equation}
Consider the sum over $i_1, \ldots, i_p$ and let us perform it first over the $p$-tuples in $\mathcal D$ and then over the $p$-tuples in $\mathcal D^{\rm c}$. Then, \eqref{eq:q1} is equal to  
\begin{equation}\label{eq:q1bis}
\sum_{k=m}^{\lfloor \ell/(p-1)\rfloor}\hspace{-1em}  c_k \sum_{i\in [r]}\left \langle \bw_{i}, \bx_j \right \rangle^{p(\ell-(p-1)k)} + \sum_{k=m}^{\lfloor \ell/(p-1)\rfloor}\hspace{-1em}  c_k \hspace{-1em}\sum_{(i_1, \ldots, i_p)\in \mathcal D^{\rm c}} \prod_{\substack{b, c\in [p]\\b<c}}\left \langle \bw_{i_{b}}, \bw_{i_c} \right \rangle^k \prod_{q = 1}^p\left \langle \bw_{i_q}, \bx_j \right \rangle^{\ell-(p-1)k}.
\end{equation}
The first term in the RHS of \eqref{eq:q1bis} is equal to $y(\bx_j)$, where $\sigma$ is given by \eqref{eq:sigmahp}. We set the error term $E(\bx)$ to the second term in the RHS of \eqref{eq:q1bis} and we bound it as follows: 
\begin{equation}
\begin{split}
\Bigg|\sum_{k=m}^{\lfloor \ell/(p-1)\rfloor}  &c_k \sum_{(i_1, \ldots, i_p)\in \mathcal D^{\rm c}} \prod_{\substack{b, c\in [p]\\b<c}}\left \langle \bw_{i_{b}}, \bw_{i_c} \right \rangle^k \prod_{q = 1}^p\left \langle \bw_{i_q}, \bx_j \right \rangle^{\ell-(p-1)k}\Bigg| \\
& \le \sum_{k=m}^{\lfloor \ell/(p-1)\rfloor}  c_k \sum_{(i_1, \ldots, i_p)\in \mathcal D^{\rm c}} \prod_{\substack{b, c\in [p]\\b<c}}\left|\left \langle \bw_{i_{b}}, \bw_{i_c} \right \rangle^k \right|\prod_{q = 1}^p\left|\left \langle \bw_{i_q}, \bx_j \right \rangle^{\ell-(p-1)k}\right|\\
& \stackrel{\mathclap{\mbox{\footnotesize (a)}}}{\le} \sum_{k=m}^{\lfloor \ell/(p-1)\rfloor}\hspace{-1em}  c_k \hspace{-.5em}\sum_{(i_1, \ldots, i_p)\in \mathcal D^{\rm c}} \hspace{-1em}\delta^{k(p-1)} \prod_{q = 1}^p\left|\left \langle \bw_{i_q}, \bx_j \right \rangle^{\ell-(p-1)k}\right|\\
& \le \sum_{k=m}^{\lfloor \ell/(p-1)\rfloor}\hspace{-1em}c_k\cdot \delta^{k(p-1)}\hspace{-1em} \sum_{i_1, \ldots, i_p\in [r]} \prod_{q = 1}^p\left|\left \langle \bw_{i_q}, \bx_j \right \rangle^{\ell-(p-1)k}\right|\\
& =\sum_{k=m}^{\lfloor \ell/(p-1)\rfloor}c_k\cdot \delta^{k(p-1)} \left(\sum_{i\in [r]} \left| \left \langle \bw_{i}, \bx_j \right \rangle^{\ell-(p-1)k}\right|\right)^p\\
& \stackrel{\mathclap{\mbox{\footnotesize (b)}}}{\le}  \sum_{k=m}^{\lfloor \ell/(p-1)\rfloor}c_k\cdot \delta^{k(p-1)} r^{p-1} \sum_{i\in [r]}  \left \langle \bw_{i}, \bx_j \right \rangle^{p(\ell-(p-1)k)}\\
& \le (\delta^{m} \cdot r)^{p-1} \sum_{k=m}^{\lfloor \ell/(p-1)\rfloor}c_k  \sum_{i\in [r]}  \left \langle \bw_{i}, \bx_j \right \rangle^{p(\ell-(p-1)k)},
\end{split}
\end{equation}
where in (a) we use that the weights $\{\bw_i\}_{1\le i\le r}\in\mathcal S_\delta$ and the fact that, for any $(i_1, \ldots, i_p)\in \mathcal D^{\rm c}$, there are at least $p-1$ pairs of distinct indices $i_b\neq i_c$, and in (b) we use H\"older's inequality.

Hence, the quantity \eqref{eq:q0} is equal to $y_{\rm noisy}(\bx_j)$, where the error term $E(\bx)$ satisfies the condition \eqref{eq:err}. Consequently, we can construct the set $\{(\bx_j, y_{\rm noisy}(\bx_j))\}_{1\le j\le n}$ with polynomial complexity. By applying the algorithm $\mathcal A$ with input $\{(\bx_j, y_{\rm noisy}(\bx_j))\}_{1\le j\le n}$, we obtain with polynomial complexity the estimates $\{\hat{\bw}_{i}\}_{1\le i \le R}$ with unit norm s.t. $|\left \langle \bw_i, \hat{\bw}_j\right \rangle| \ge \epsilon$ for some $i \in [r]$ and $j\in [R]$. As a result, given the tensor $\bT^{(\ell)}$, the problem of learning $\{\bw_{i}\}_{1\le i \le r}$ is not $\epsilon$-hard, which violates the hypothesis and concludes the proof.
\end{proof}

\end{document}